\documentclass[11pt]{article}
\usepackage{fullpage}
\pdfoutput=1
\usepackage{lipsum} % for dummy text

\usepackage{twoopt}
\usepackage[numbers]{natbib}
\usepackage{subfigure}

\usepackage{fullpage}
\usepackage{float}

\usepackage{times}
\usepackage{latexsym}
\usepackage{amsmath}
\usepackage{amssymb}
\usepackage{amsfonts}
\usepackage{mathrsfs}
\usepackage{aliascnt}
\usepackage{graphicx}
\usepackage{epstopdf}
\usepackage[dvipsnames]{xcolor}
\usepackage{color}
\usepackage{enumerate}
\usepackage{array}
\usepackage{enumitem}

\usepackage[hypertexnames=false,colorlinks=true,pdfpagemode=Usenone,urlcolor=Blue,linkcolor=RoyalBlue,citecolor=OliveGreen,pdfstartview=FitH]{hyperref}
\usepackage{algorithmic}
\usepackage{algorithm}

\newtheorem{theorem}{Theorem}

\newtheorem{claim}{Claim}

\newtheorem{remark}{Remark}
\newcommand{\qed}{\mbox{\ \ \ }\rule{6pt}{7pt} \bigskip}

\newenvironment{proof}{\noindent{\em Proof:}}{\hfill\qed}

\makeatletter
\newenvironment{oneshot}[1]{\@begintheorem{#1}{\unskip}}{\@endtheorem}
\makeatother

\newcommand{\AutoAdjust}[3]{\mathchoice{ \left #1 #2  \right #3}{#1 #2 #3}{#1 #2 #3}{#1 #2 #3} }
\newcommand{\Xcomment}[1]{{}}

\newcommand{\InBrackets}[1]{\AutoAdjust{[}{#1}{]}}% {\left[{#1}\right]}
\newcommand{\Ex}[2][]{\operatorname{\mathbf E}_{#1}\InBrackets{#2}}

\newcommand{\eqdef}{\stackrel{\textrm{def}}{=}}

\newcommand{\be}{\begin{equation}}
\newcommand{\ee}{\end{equation}}
\newcommand{\bee}{\begin{equation*}}
\newcommand{\eee}{\end{equation*}}

\newcommand{\eps}{\varepsilon}

\newcommand{\gain}{g}

\newcommandtwoopt{\gainit}[2][i][t]{\gain_{#1#2}}

\newcommand{\Gain}{G}

\newcommandtwoopt{\Gainit}[2][i][t]{\Gain_{#1#2}}

\newcommand{\family}{\mathcal{F}}
\newcommand{\algfam}{\mathcal{A}_\family}
\newcommand{\advfam}{\mathcal{D}_\family}
\newcommand{\algsingle}{\mathcal{A}_\text{single}}
\newcommand{\algdec}{\mathcal{A}_\text{dec}}
\newcommand{\algarb}{\mathcal{A}_\text{arb}}
\newcommand{\algrand}{\mathcal{A}_\text{rand}}
\newcommand{\alguniv}{\mathcal{A}_\text{univ}}
\newcommand{\algconv}{\mathcal{A}_\text{conv}}
\newcommand{\advuniv}{\mathcal{D}_\text{univ}}
\newcommand{\advsimple}{\mathcal{D}_\text{simple}}

\newcommand{\Loop}{\text{Loop}}

\begin{document}
\pagenumbering{gobble}
\title{Tight Lower Bounds for Multiplicative Weights Algorithmic Families}

\author{
Nick Gravin\thanks{Massachusetts Institute of Technology.
32 Vassar St, Cambridge, MA 02139. \texttt{ngravin@mit.edu}.}
\and
Yuval Peres\thanks{Microsoft Research.
One Microsoft Way, Redmond, WA 98052.
\texttt{peres@microsoft.com}.
}
\and
Balasubramanian Sivan\thanks{Google Research.
111 8th Avenue, New York, NY 10011.
\texttt{balusivan@google.com}.
}
}

\date{}

\maketitle

\begin{abstract}
We study the fundamental problem of prediction with expert advice and develop 
regret lower bounds for a large family of algorithms for this problem. We 
develop simple adversarial primitives, that lend themselves to 
various combinations leading to sharp lower bounds for many algorithmic 
families. We use these primitives to show that the classic 
Multiplicative 
Weights Algorithm (MWA) has a regret of $\sqrt{\frac{T \ln k}{2}}$, there by 
completely 
closing the gap between upper and lower bounds. We further show a regret lower 
bound 
of $\frac{2}{3}\sqrt{\frac{T\ln k}{2}}$ for a much more general family of 
algorithms than MWA, where the learning rate can be arbitrarily varied over 
time, or even picked from arbitrary distributions over time. We also use our 
primitives to construct adversaries in the geometric horizon setting for MWA to 
precisely characterize the regret at $\frac{0.391}{\sqrt{\delta}}$ for the case 
of $2$ experts and a lower bound of $\frac{1}{2}\sqrt{\frac{\ln k}{2\delta}}$ 
for the 
case of arbitrary number of experts $k$.

\end{abstract}
\newpage
\pagenumbering{arabic}
\section{Introduction}
\label{sec:intro}
In this paper we develop tight lower bounds on the regret obtainable by a broad
family of algorithms for the fundamental problem of prediction with expert
advice. Predicting future events based on past observations, a.k.a. prediction
with expert advice, is a classic problem in learning. The experts framework was
the first framework proposed for online learning and encompasses several
applications as special cases. The underlying problem is an online optimization
problem: a {\em player} has to make a decision at each time step, namely, decide
which of the $k$ experts' advice to follow. At every time $t$,  an {\em
adversary} sets gains for each expert: a gain of $\gainit$ for expert $i$ at
time $t$. Simultaneously, the \emph{player}, seeing the gains from all previous
steps except $t$, has to choose an action, i.e., decide on which expert to
follow. If the player follows expert $j(t)$ at time $t$, he gains
$\gainit[j(t),][t]$. At the end of each step $t$, the gains associated with all
experts are revealed to the player, and the player's choice is revealed to the
adversary. In the {\em finite horizon model}, this process is repeated for $T$
steps, and the player's goal is to perform (achieve a cumulative gain) as close
as possible to the best single action (best expert) in hindsight, i.e., to
minimize his \emph{regret} $R_{T,k}$: $$R_{T,k} = \max_{1\leq i \leq k}
\sum_{t=1}^T \gainit - \sum_{t=1}^T \gainit[j(t),][t].$$ Apart from assuming
that the $\gainit$'s are bounded in $[0,1]$, we don't assume anything else about
the gains\footnote{As one might expect, it turns out that
restricting the adversary to set gains in $\{0,1\}$ instead of $[0,1]$ is
without loss of generality (see~\cite{GPS16} or~\cite{LS14}). Henceforth, we
restrict ourselves to the binary adversary, which just sets gains of $0$ or
$1$.}.  Just as natural as the finite horizon model is the model with a {\em
geometric horizon}: the stopping time is a geometric random variable with
expectation $\frac{1}{\delta}$. In other words, the process ends at any given
step with probability $\delta$, independently of the past. Equivalently, both
the player and the adversary discount the future with a $1-\delta$ factor. In
this paper, we study both the finite horizon model and the geometric horizon
model. We begin with the discussion for finite horizon model below. 
\vspace{-0.6em}
\paragraph{Main contribution.} In this paper we develop simple adversarial
primitives and demonstrate that, when applied in various combinations, they
result in remarkably sharp lower bounds for a broad family of algorithms. We
first describe the family of algorithms we study, and then discuss our main
results.
\vspace{-0.6em}
\paragraph{Multiplicative Weights Algorithm.} We begin with the Multiplicative
Weights Algorithm, which is a simple, powerful and widely used algorithm for a
variety of learning problems. In the experts problem, at each time $t$, MWA
computes the cumulative gain $\Gainit[i][t-1] = \sum_{s=1}^{t-1} \gainit[i][s]$
of each expert $i$ accumulated over the past $t-1$ steps, and will follow expert
$i$'s advice with probability proportional to $e^{\eta\Gainit[i][t-1]}$. Namely,
with probability $\frac{e^{\eta\Gainit[i][t-1]}}{{\sum_{j=1}^{k} e^{\eta
\Gainit[j][t-1]}}}$ where $\eta$ is a parameter that can be tuned.  The per-step
computation of the algorithm is extremely simple and straightforward. The
intuition behind the algorithm is to increase the weight of any expert that
performs well by a multiplicative factor.  Despite the simplicity and the
heuristic origins of the algorithm, it is surprisingly powerful: the pioneering
work of Cesa Bianchi et al.~\cite{BFHHSW97} showed that MWA obtains a sublinear
regret of $\sqrt{\frac{T\ln k}{2}}$, and that this is asymptotically optimal as
the number of experts $k$ and the number of time steps $T$ both tend to
$\infty$.
\vspace{-0.4em}
\paragraph{Families of algorithms.} The MWA is a single-parameter family of
algorithms, i.e., the learning rate parameter $\eta$ is the only parameter
available for the player. In general one could think of $\eta$ being an
arbitrary function of time $t$, i.e., at step $t$, algorithm follows expert $i$
with probability $\frac{e^{\eta(t)\Gainit[i][t-1]}}{{\sum_{j=1}^{k} e^{\eta(t)
\Gainit[j][t-1]}}}$. Note that this is a $T$-parameter family of algorithms and
is quite general. To see why this is general, note that after fixing 
$\Gainit[i][t-1]$ for all 
$i$,
\emph{any probability $p_{it}$ of picking expert $i$ at time $t$ can expressed 
as}
$\frac{e^{\eta(t)\Gainit[i][t-1]}}{{\sum_{j=1}^{k} e^{\eta(t)
\Gainit[j][t-1]}}}$, irrespective of what $p_{i,t-1}$ was --- something that is 
certainly not possible when $\eta$ is independent of $t$. The most general 
family of algorithms we study is when at each time $t$, the quantity $\eta(t)$ 
is drawn from an arbitrary 
distribution $F_t$ over reals.
Since $F_t$ could be arbitrary, this is an infinite-parameter
family of algorithms. We denote the 
\begin{enumerate}
\itemsep-0.3em 
\item single parameter MWA
family by $\algsingle$; 
\item family where $\eta(t)$ decreases with $t$ by
$\algdec$; 
\item family where $\eta(t)$ is arbitrary function of $t$ by
$\algarb$; 
\item family where $\eta(t)$ is drawn from $F_t$ for each $t$ by
$\algrand$. 
\end{enumerate} 
It is straightforward to see that $\algsingle
\subseteq \algdec \subseteq \algarb \subseteq \algrand$. The reason we start 
with $\algsingle$ is that it is the classic MWA and precisely characterizing 
its regret is still open. We study $\algdec$ because often when MWA algorithms 
are working with unknown $T$, they employ a strategy where $\eta$ decreases 
with time. We move on to further significantly generalize this by studying 
$\algarb, \algrand$.
\vspace{-0.4em}
\paragraph{Minimax regret, and Notation.} We study the standard notion of 
minimax regret for
each of the above family of algorithms. Formally, let $R_{T,k}(A,D)$ denote the
regret achieved by algorithm $A$ when faced with adversary $D$ in the prediction
with expert advice game with $T$ steps and $k$ experts. We use $R_k(A,D)$ to 
denote the asymptotic\footnote{Although $R_k$ doesn't 
have a $T$ in the subscript, $R_k$ is still dependent on $T$. We suppress $T$ 
merely to indicate asymptotics in $T$.},
in 
$T$, value of $R_{T,k}(A,D)$, i.e., $R_k(A,D)  = 
\sqrt{T}\cdot\lim_{T\to\infty}\frac{R_{T,k}(A,D)}{\sqrt{T}}$. 
The minimax regret of a family $\algfam$ of algorithms against a family 
$\advfam$ of adversaries is given by $R_{T,k}(\algfam, \advfam) = \min_{A \in 
\algfam}\max_{D(A) \in \advfam} R_{T,k}(A,D)$. Let $\advuniv$ denote
the universe of all adversaries. We use the shorthand $R_{T,k}(\algfam)$ for 
$R_{T,k}(\algfam,\advuniv) = \min_{A \in \algfam}\max_{D(A) \in
\advuniv} R_{T,k}(A,D)$. We use $R_k(\algfam)$ to denote the asymptotic, in 
$T$, value of $R_{T,k}(\algfam)$, i.e., $R_k(\algfam)  = 
\sqrt{T}\cdot\lim_{T\to\infty}\frac{R_{T,k}(\algfam)}{\sqrt{T}}$. 
%$\lim_{T\to\infty}R_{T,k}(\algfam)$. 
\vspace{-0.4em}
\paragraph{Goal.} One of our goals in this paper is to compute the precise 
values of
$R_{k}(\algsingle)$, $R_{k}(\algdec)$, $R_{k}(\algarb)$ and $R_{k}(\algrand)$
for each value of $k$, and, \emph{describe and compute} the adversarial 
sequences that realize these
regrets. For clarity, we compute the precise values of $R_k(\algfam)$ by: 
\begin{enumerate}
\itemsep-0.3em 
\item computing the best-response adversary in $\advuniv$ for every algorithm 
in $\algfam$;
\item computing $R_k(\algfam)$ the regret of the optimal algorithm in $\algfam$ (i.e., 
the algorithm that gets the smallest regret w.r.t. its best-response adversary).
% the regret of the optimal algorithm in the second step is precisely 
% $R_k(\algfam)$.  
\end{enumerate}
In many cases, the first step, namely computing the best-response adversary, is
challenging. We find the best-response adversaries for the families
$\algsingle$ and $\algdec$. For the families $\algarb$ and $\algrand$, we 
perform the first step approximately, i.e., we compute a nearly best-response 
adversary, and thus we obtain lower bounds on $R_k(\algarb)$ and 
$R_k(\algrand)$. 
\vspace{-0.4em} 
\paragraph{What is known, and what to expect?} It is well known that for 
$\algfam=\algsingle,\algdec,\algarb,\algrand$: $R_{T,k}(\algfam) \leq 
\sqrt{(T\ln k)/2}$ for all $T,k$, and in the doubly 
asymptotic limit, as both $T$ and $k$ go to $\infty$, the optimal regret of 
$\algsingle$ is $\sqrt{(T\ln k)/2}$, i.e., 
$\lim\limits_{T\to\infty, 
k\to\infty}\left(R_{T,k}(\algsingle)/\sqrt{(T\ln k)/2}\right) = 1$.
(see~\cite{BFHHSW97, Bianchi99}). While there are useful applications for 
$k\to 
\infty$, there are also several interesting use-cases of 
the experts problem with just a few experts (rain-or-shine ($k=2$), 
buy-or-sell-or-hold ($k=3$)). It seems like for small $k$ such as $2,3,4$ etc. 
$R_k(\algsingle)$ could be a significant constant factor smaller than 
$\sqrt{(T\ln k)/2}$. And given that families like $\algdec$ 
etc. are supersets of 
$\algsingle$, it seems even more likely that $R_k(\algdec)$ etc. are 
constant factor smaller than $\sqrt{(T\ln k)/2}$. Surprisingly, we show 
that is not the case: the regret of $\sqrt{(T\ln k)/2}$ that is 
obtained as $k\to\infty$ is \emph{already obtained at $k=2$}. Thus our work 
completely closes the gap between upper and lower bounds for all $k$. 
\vspace{-0.4em}
\subsection{Main Results}
\paragraph{Finite horizon model.} 
\begin{enumerate}[leftmargin=*]
\itemsep-0.3em 
\item $R_k(\algsingle) = R_k(\algdec) = \sqrt{\frac{T\ln k}{2}}$ for even 
$k$,\quad
$R_k(\algsingle) \ge R_k(\algdec)\ge \sqrt{\frac{T\ln k}{2}(1-\frac{1}{k^2})}$ for odd 
$k$.
% \item $R_k(\algdec) = \sqrt{\frac{T\ln k}{2}}$ for every even $k$ and 
% $R_k(\algdec) = \sqrt{\frac{T\ln k}{2}(1-\frac{1}{k^2})}$ for every odd $k$. 
\item $R_k(\algarb)\ge R_k(\algrand)\ge \frac{2}{3}\sqrt{\frac{T\ln k}{2}}$ for even $k$,\quad
$R_k(\algarb) \geq R_k(\algrand)\ge \frac{2}{3}\sqrt{\frac{T\ln k}{2}(1-\frac{1}{k^2})}$ for
odd $k$.
% \item $R_k(\algrand) \geq \frac{2}{3}\sqrt{\frac{T\ln k}{2}}$ for every even
% $k$ and $R_k(\algrand) \geq \frac{2}{3}\sqrt{\frac{T\ln k}{2}(1-\frac{1}{k^2})}$
% for every odd $k$.
\end{enumerate}
\vspace{-0.4em}
\paragraph{Geometric horizon model.} In the geometric horizon model, the current
time $t$ is not relevant, since the expected remaining time for which the game
lasts is the same irrespective of how many steps have passed in the past. Thus
$\eta(t)$ is without loss of generality, independent of $t$. Nevertheless,
$\eta$ could still depend on other aspects of the history of the game, like the
cumulative gains of all the experts etc. We establish some quick notation before
discussing results. Let $\delta$ denote the probability that the game stops at
any given step, independently of the past (and therefore the expected length of
the game is $\frac{1}{\delta}$). Let $R_{\delta,k}(A,D)$ denote the regret
achieved by algorithm $A$ when faced with adversary $D$ in the prediction with
expert advice game with stopping probability $\delta$ and $k$ experts. %Let
% $\alguniv$ and $\advuniv$ denote respectively the universe of all
%algorithms and adversaries.
The minimax regret for a family $\algfam$ of algorithms is given by
$R_{\delta,k}(\algfam) = R_{\delta,k}(\algfam,\advuniv)=\min_{A \in 
\algfam}\max_{D \in \advuniv}
R_{\delta,k}(A,D)$. Let\footnote{Note that the notation $R_k(\cdot)$ is 
overloaded: it could refer to finite or geometric horizon setting depending on 
the context. 
But since the setting is clear from the context, we drop the $\delta$ vs $T$.} 
$R_k(\algfam) = \frac{1}{\sqrt{\delta}}\lim_{\delta\to 
0}\sqrt{\delta}\cdot R_{\delta,k}(\algfam)$.

We show the following:
% \[
% R_2(\algsingle) = \frac{0.391}{\sqrt{\delta}},\quad\quad\quad\quad
% R_k(\algsingle) \geq \frac{1}{2}\sqrt{\frac{\ln k}{2\delta}} \text{ for all }k.
% \]
% 
% \begin{enumerate}
% \item $R_2(\algsingle) = \frac{0.391}{\sqrt{\delta}}$.
% \item $R_k(\algsingle) \geq \frac{1}{2}\sqrt{\frac{\ln k}{2\delta}}$ for all 
% $k$.
% \end{enumerate}
\quad\quad(1) $R_2(\algsingle) = \frac{0.391}{\sqrt{\delta}}$; \quad\quad (2)
$R_k(\algsingle) \geq \frac{1}{2}\sqrt{\frac{\ln k}{2\delta}}$ for all 
$k$.
%In the first result we \emph{precisely pin down} the regret obtained by the
%family $\algsingle$ in the $2$ experts case to $\frac{0.391}{\sqrt{\delta}}$; 
%and, in
%the second result, we show a \emph{regret lower bound} for all $k$ for the
%$\algsingle$ family. 

The regret lower bound of $\frac{1}{2}\sqrt{\frac{\ln
k}{2\delta}}$ we obtain is at most a factor $2$ away from the regret 
upper bound of $\sqrt{\frac{\ln k}{2\delta}}$. Further, we show that the 
adversarial 
family that we use for the family of algorithms $\algsingle$ to obtain the 
precise regret for $2$ experts, also obtains the optimal regret for the 
universe $\alguniv$ of \emph{all algorithms}. See 
Remark~\ref{rem:geom_2_general} for more on 
this 
result. 
%To see that 
%this is the upper bound consider the quantity 
%$\sum_{T=0}^{\infty}\sqrt{\frac{T\ln
%k}{2}}(1-\delta)^T\delta \sim \frac{\sqrt{\pi}}{2}\sqrt{\frac{\ln 
%k}{2\delta}}$.) 
\vspace{-0.4em}
\subsection{Simple adversarial primitives and families}
\label{subsec:SAPF} While the optimal regret $R_k(\algfam)$ 
is defined by optimizing over the most general family $\advuniv$ of all 
adversaries, (i.e., $R_k(\algfam) = R_k(\algfam,\advuniv)$) one of our primary 
contributions in this work is to develop simple 
and 
analytically easy-to-work-with adversarial primitives that we use to construct 
adversarial families (call a typical such family 
$\advsimple$) such that:
\vspace{-0.4em}  
\begin{itemize}
\itemsep-0.3em 
\item $\advsimple$ is simple to-describe and to-optimize-over, i.e., 
computing $\max_{{D \in \advsimple} R_{T,k}(A,D)}$ is much simpler than computing $\max_{D \in
\advuniv} R_{T,k}(A,D)$.
% $\advsimple$ is easy-to-optimize-over (unlike $\advuniv$), i.e., 
% computing $\max_{{D \in
% \advsimple} R_{T,k}(A,D)}$ is much simpler than computing $\max_{D \in
% \advuniv} R_{T,k}(A,D)$.
\item optimizing over $\advsimple$ is guaranteed to be as good (or 
approximately as good) as optimizing over $\advuniv$ for many algorithmic 
families $\algfam$, i.e., $R_k(\algfam,\advuniv) = R_k(\algfam,\advsimple)$ for 
many $\algfam$. As $R_k(\algfam,\advuniv) \geq 
R_k(\algfam,\advsimple)$, the non-trivial part is to prove (approximate) equality
for $\algfam$. 
\end{itemize}
We demonstrate the versatility of our primitives by using simple 
combinations of them to develop sharp 
lower bounds to algorithmic families $\algsingle$, $\algdec$, 
$\algarb$, and $\algrand$. 
%These primitives have room for various kinds of 
%combinations that one could use to construct adversarial families that suit 
%different algorithmic families. 
There is a lot of room for further combinations of primitives that might be 
useful to construct adversarial families tailored to other algorithmic families. 
\vspace{-0.4em}
\paragraph{The ``looping'' and ``straight-line'' primitives.} These primitives 
are best described by focusing on the case of $k=2$ experts. In the two experts 
case, the algorithm makes its decision at step $t$, by just looking at the 
difference $d$ of the cumulative gains of the leading and lagging experts' 
cumulative gains. As such, the adversary has to simply control how the 
difference $d$ evolves over time. The ``looping'' primitive simply loops the 
value of $d$ between $0$ and $1$ indefinitely (i.e., advances\footnote{advances 
here refers to setting the gain of that expert to $1$.} one expert in one 
step and advances the other in the next step and so on, so that $d$ simply 
loops between $0$ and $1$). The ``straight-line'' 
adversary simply keeps advancing the value of $d$ by $1$ at each step. 
Interestingly, the worst-case adversary for each of the finite and geometric 
horizon settings is a composition of looping and straight-line primitives. Strikingly, despite the
apparent similarity between two settings, the
optimal adversaries in the two models turn out to be ``mirror images'' of each
other. The optimal adversary in finite horizon loops first and then goes in 
straight-line, while the geometric horizon's optimal does the reverse. The 
structure of these two families is depicted 
in Figure~\ref{fig:char}, that shows the evolution of the difference $d$ 
between the cumulative gains of the leading and lagging experts. This 
fundamental difference between the structures of the optimal adversary in these 
two settings also manifests in the optimal regret values of these two settings 
(see 
Remark~\ref{rem:GeomVsFinite}). 

The generalizations of these primitives for arbitrary $k$ is 
straightforward. The looping primitive partitions the set of experts into two 
teams, say $A$ and $B$, and then it advances all experts in team $A$ in one 
step and in team $B$ in the other, and so on. The straight-line primitive picks 
an arbitrary expert and keeps advancing that expert by $1$ in each step. 

\begin{figure}
\centering
\begin{subfigure}
	\centering
	\includegraphics[scale=.4]{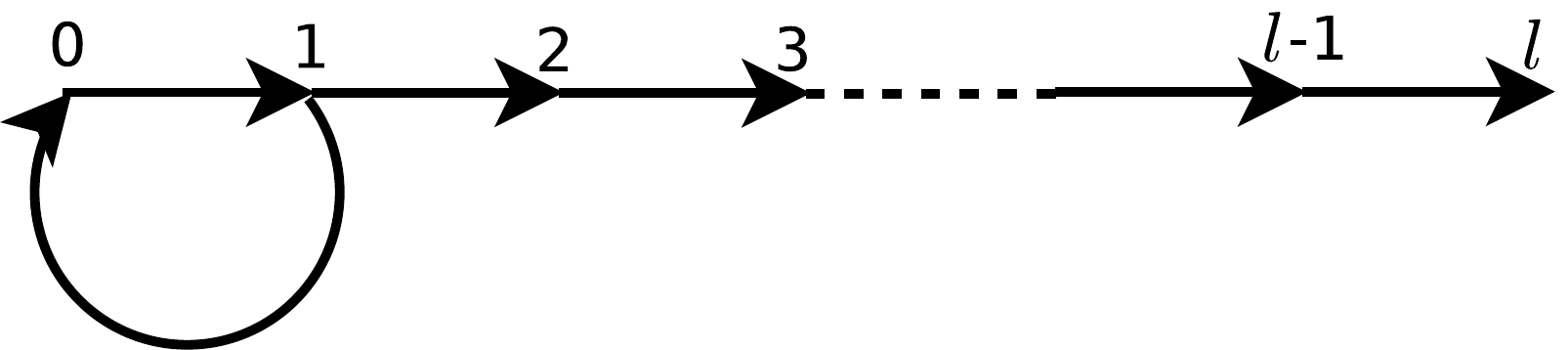}
	%\caption{Optimal geometric adversary}
\end{subfigure}
\ \ \ \ \ \ \ \ \ \ \ \ \ \ \ \ \qquad
\begin{subfigure}
	\centering
	\includegraphics[scale=.4]{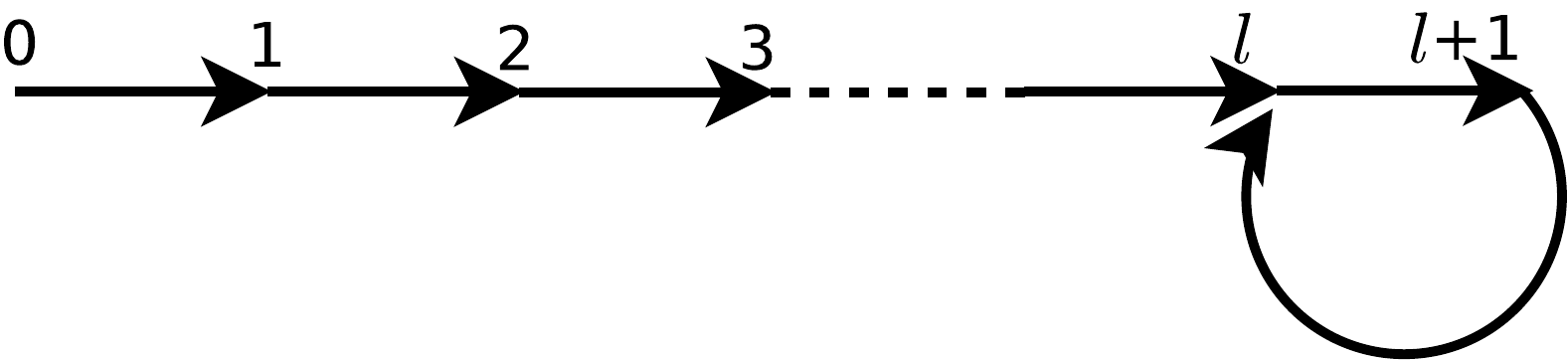}
	%\caption{Optimal finite adversary}
\end{subfigure} 
\caption{Optimal adversaries in finite horizon (left) and geometric horizon 
(right)} 
\label{fig:char}
\end{figure}

\vspace{-0.4em}
\paragraph{Combining the primitives.} Here's how we create effective 
adversarial families from these primitives. In fact the families are often 
trivial, 
i.e., they have only one member and therefore there's nothing to optimize. 
We ignore the odd and even $k$ 
distinctions here for ease of description and just focus on the even $k$ case. 
Please see the technical sections 
for precise descriptions, which is only slightly different from what is 
here. 
\vspace{-0.4em}
\begin{enumerate}[leftmargin=*]
\itemsep-0.3em 
\item Perform $\frac{T-\ell}{2}$ loops and then $\ell$ straight-line steps, for 
$\ell = T^{3/4}$. Call this adversary $D_{lsdet}$ (stands for 
loop-straight-deterministic). Clearly, this adversarial 
family is simple-to-describe and there 
is nothing to optimize here as there is only one member in the family. Most 
importantly, it gives 
the precisely 
optimal regret for algorithmic families $\algsingle$ and $\algdec$ as $T 
\to\infty$. I.e.,
%\[
%R_{k}(\algsingle,\advuniv) = R_{k}(\algsingle,D_{lsdet}) 
%= \sqrt{\frac{T\ln k}{2}} = R_{k}(\algdec,D_{lsdet}) =
%R_{k}(\algdec,\advuniv).
%\]
\[
\text{For families $\algfam = \algsingle, \algdec$:\qquad} 
R_{k}(\algfam,D_{lsdet}) 
= \sqrt{\frac{T\ln k}{2}} = R_{k}(\algfam,\advuniv).
\]
The best known regret lower bound for $\algsingle$ was 
$\frac{1}{4}\sqrt{T\log_2 k}$~\cite{GPS13}, which leaves a factor $2.35$ gap 
between upper and lower bounds, that our work closes. 
%There was no known lower 
%bound for $\algdec$. \nick{There might be some known lower bound.}
We are not aware of prior lower bounds for $\algdec$.
\item Perform $\frac{T-r}{2}$ loops and then $r$ straight-line steps, where $r$ 
is chosen uniformly at random from $\{0,1,\dots,T^{3/4}\}$. This family is 
simple and there is nothing to optimize here as well. Call 
this adversary $D_{lsrand}$ (denoting loop, straight, uniformly random). 
We show that when $\algfam = \algarb$ or when $\algfam = \algrand$: 
%\[
%R_{k}(\algarb, D_{lsrand}) \ge R_{k}(\algrand,D_{lsrand}) 
%\ge \frac{2}{3}\sqrt{\frac{T\ln k}{2}} \ge 
%\frac{2}{3}R_{k}(\algarb, \advuniv) \ge \frac{2}{3}R_{k}(\algrand, \advuniv).
%\]
\[
\text{For families $\algfam = \algarb, \algrand$:\qquad} 
R_{k}(\algfam, D_{lsrand}) \ge \frac{2}{3}\sqrt{\frac{T\ln k}{2}} \ge
\frac{2}{3}R_{k}(\algfam, \advuniv).
\]
Note that while this lower bound doesn't precisely match the upper bound, 
the upper bound $R_k(\algfam,\advuniv) \leq \sqrt{(T\ln 
k)/2}$ and is 
likely even smaller for small $k$ (particularly for a large family of 
algorithms like $\algarb$ or $\algrand$) --- \emph{thus our result shows that 
the ratio between upper 
and lower bounds is at most 
$\frac{3}{2}$} and likely even smaller. 
%There was no known regret lower bounds for $\algarb$ and $\algrand$ earlier. 
To the best of our knowledge our lower bound is the first for the classes 
$\algarb$ and $\algrand$.
%There was no known lower bound earlier for these general families. 
\item In geometric horizon, even for the family $\algsingle$ and 
at $k=2$ experts, instead of a single adversary 
working for all members of $\algsingle$, we have a single-parameter family of 
adversaries to optimize over. Namely, follow the straight-line primitive for 
$r$ steps and then the looping primitive for $\frac{T-r}{2}$ steps. Call this 
single-parameter family (parameterized by $r$) as 
$\mathcal{D}_{sl}$. 
The exact number $r$ is determined by optimizing it as a function of the 
parameter $\eta$ used by the algorithm in $\algsingle$. Specifically, for 
the case of $2$ experts we show that: 
$R_2(\algsingle, \mathcal{D}_{sl}) 
= 
\frac{0.391}{\sqrt{\delta}} = R_2(\algsingle,\advuniv).$
%$$\lim_{\delta\to0} \min_{A \in \algsingle}\max_{D\in\advuniv} R_{T,k} (A,D) 
%\sim \lim_{\delta\to0} \min_{A \in \algsingle}\max_{D\in 
%D_{s,\ell,single}}R_{\delta,2}(A,D) 
%\sim \frac{0.391}{\sqrt{\delta}}.$$
Note that $D_{sl}$ is again simple-to-describe and straightforward-to-optimize over. Further, it is the 
precisely optimal adversary family for not just $\algsingle$ but also the 
\emph{universe of all algorithms} $\alguniv$ (see 
Remark~\ref{rem:geom_2_general}), 
i.e., $R_2(\alguniv,\advuniv) = R_2(\alguniv, \mathcal{D}_{sl}).$
%$$\lim_{\delta\to 0} \min_{A \in \alguniv}\max_{D\in\advuniv} R_{T,k} (A,D) 
%\sim \lim_{\delta\to 0} \min_{A \in \alguniv}\max_{D\in 
%D_{s,\ell,single}}R_{\delta,2}(A,D).$$
\item But in the geometric horizon setting, if we don't shoot for the 
precisely optimal adversary family, and aim for just approximately optimal, 
then we 
don't need a single-parameter family: just 
following one of the two 
looping/straight-line primitives gives a lower bound of 
$\frac{1}{2}\sqrt{\frac{\ln k}{2\delta}}$. 
Let $D_{\ell}, D_{s}$ be the looping and straight line primitives. Then:
$R_k(\algsingle,\{D_\ell,D_s\}) \geq 
\frac{1}{2}\sqrt{\frac{\ln k}{2\delta}} \geq 
\frac{1}{2}R_k(\algsingle,\advuniv).$
%$$\lim_{\delta\to 
%0}\min_{A\in\algsingle}\max_{D\in\{D_{\ell},D_s\}}R_{\delta,k}(A,D) \geq 
%\frac{1}{2}\sqrt{\frac{\ln k}{2\delta}}.$$
Note that while this lower bound doesn't precisely match the upper bound 
$R_k(\algsingle,\advuniv)$, the 
latter is at most\footnote{This is a simple extension of the 
standard proof that MWA has a 
regret upper bound of $\sqrt{(T\ln k)/2}$ in the finite horizon 
setting with $T$ steps and $k$ experts, and the realization that in the 
geometric horizon setting, the expected stopping time is $\frac{1}{\delta}$.} 
$\sqrt{\frac{\ln 
k}{2\delta}}$, which is at most a factor $2$ larger than lower bound. The only 
known regret lower bounds in the 
geometric horizon setting was what one could infer from the finite horizon 
setting lower 
bound of $\frac{1}{4}\sqrt{T\log_2 k}$~\cite{GPS13}, and it is not even clear 
what this 
exactly 
translates to in the geometric horizon setting. 
\end{enumerate}

\begin{remark}To give a sense that the primitives offer enough variety in 
combination, here is a simple modification over 
the adversary $D_{lsrand}$, 
that we call $D_{lsrand++}$: 
use $D_{lsrand}$ with 
probability $p$, 
and with probability $1-p$ play the looping primitive $D_{l}$ for all the $T$ steps. 
This increases the lower bound from $\frac{2}{3}\sqrt{(T\ln k)/2}$ to 
$0.68\sqrt{(T\ln k)/2}$ (see Theorem~\ref{thm:kExpertsLB_var}). We believe that 
this can be 
increased further 
by picking the stopping time for looping from a non-uniform distribution etc.
\end{remark}

\subsection{Motivation and discussion} In this work we seek to understand the 
structure of worst case input sequences for a broad family of algorithms and 
crisply expose 
their vulnerabilities. By identifying such structures, we also get the
precise regret suffered by them.  Our motivation in exploring this question
includes the following. 
\vspace{-0.4em}
\begin{enumerate}[leftmargin=*]
\itemsep-0.3em 
\item After 25 years since MWA was introduced~\cite{LW94, Vovk90}, we do
not have a sharp regret bound for it. $\algsingle$ is known to suffer a regret 
of at
most $\sqrt{\frac{T\ln k}{2}}$, but the best known lower bound on regret is
$\frac{1}{4}\sqrt{T\log_2 k}$~\cite{GPS13}, with a factor $2.35$ gap between
these two bounds. For larger families like $\algarb$, $\algrand$ no lower 
bounds were known. For an algorithm as widely used as MWA, it is fruitful to
have a sharp regret characterization. 
%We do this for a much broader family than MWA. 
\item The patterns in the worst-case adversarial sequences that we
characterize are simple to spot if they exist (or even if anything close
exists), and make simple amends to the algorithm that result in 
significant gains.  
\item The problem is theoretically clean and challenging: how powerful are 
simple input patterns beyond the typically used pure random sequences in 
inflicting
regret?
\end{enumerate}

\paragraph{Related Work.} 
\textit{Classic works:} The book by~\citet{BL06-book} is an excellent source
for both applications and references for prediction with expert advice. The
prediction with experts advice paradigm was introduced by~\citet{LW94}
and~\citet{Vovk90}. The famous multiplicative weights update algorithm was
introduced independently by these two works: as the weighted majority algorithm
by~\citeauthor{LW94} and as the aggregating algorithm by~\citeauthor{Vovk90}.
The pioneering work of~\citet{BFHHSW97} considered $\{0,1\}$ outcome space for
nature and showed that for the absolute loss function $\ell(x,y) = |x-y|$ (or
$g(x,y) = 1-|x-y|$), the asymptotically optimal regret is $\sqrt{\frac{T\ln
k}{2}}$. This was later extended to $[0,1]$ outcomes for nature
by~\citet{HKW95}. The asymptotic optimality of $\sqrt{\frac{T\ln k}{2}}$ for
arbitrary loss (gain) functions follows from the analysis of~\citet{Bianchi99}.
When it is known beforehand that the cumulative loss of the optimal expert is
going to be small, the optimal regret can be considerably improved, and such
results were obtained by~\citet{LW94} and~\citet{FS97}. With certain
assumptions on the loss function, the simplest possible algorithm of following
the best expert already guarantees sub-linear regret~\citet{Hannan57}. Even
when the loss functions are unbounded, if the loss functions are exponential
concave, sub-linear regret can still be achieved~\citet{BK99}.

\textit{Recent works:}~\citet{GPS16} give the minimax optimal algorithm,
and the regret for the prediction with expert advice problem for the
cases of $k=2$ and $k=3$ experts. The focus of~\cite{GPS16} was providing a 
regret 
upper bound for the family of all algorithms, while the focus of this paper is 
to provide regret lower bounds for large families of algorithms.~\citet{LS14} 
consider a setting where the
adversary is restricted to pick gain vectors from the basis vector space
$\{\mathbf{e}_1,\dots,\mathbf{e}_k\}$.~\citet{AbernethyWY08} consider a
different variant of experts problem where the game stops when cumulative loss
of any expert exceeds given threshold. 
~\citet{ABRT08} consider general convex games and compute the minimax regret
exactly when the input space is a ball, and show that the algorithms
of~\citet{Zinkevich03} and~\citet{HKKA06} are optimal w.r.t. minimax
regret.~\citet{AABR09} provide upper and lower bounds on the regret of an
optimal strategy for several online learning problems without providing
algorithms, by relating the optimal regret to the behavior of a certain
stochastic process.~\citet{MS10} consider a continuous experts setting where
the algorithm knows beforehand the maximum number of mistakes of the best
expert.~\citet{RST10} introduce the notion of sequential Rademacher complexity
and use it to analyze the learnability of several problems in online learning
w.r.t. minimax regret.~\citet{RST11} use the sequential Rademacher complexity
introduced in~\cite{RST10} to analyze learnability w.r.t. general notions of
regret (and not just minimax regret).~\citet{RSS12} use the notion of
conditional sequential Rademacher complexity to find relaxations of problems
like prediction with static experts that immediately lead to algorithms and
associated regret guarantees. They show that the random playout strategy has a
sound basis and propose a general method to design algorithms as a random
playout. 
~\citet{Koolen13} studies the regret w.r.t. every expert, rather than just the
best expert in hindsight and considers tradeoffs in the
Pareto-frontier.~\citet{MA13} characterize the minimax optimal regret for
online linear optimization games as the supremum over the expected value of a
function of a martingale difference sequence, and similar characterizations for
the minimax optimal algorithm and the adversary.~\citet{MO14} study online
linear optimization in Hilbert spaces and characterize minimax optimal
algorithms.~\citet{CFH09} describe a parameter-free learning algorithm 
motivated 
by the cases of large number of experts $k$.~\citet{KE15} develop a prediction 
strategy called Squint, and prove bounds that incorporate both quantile and 
variance guarantees. 

%\section{Preliminaries}
%\label{sec:prelim}
%\input{prelim}

\section{Finite horizon}
\label{sec:finite}
We begin our analysis of MWA by focusing on the simple case of $k=2$ experts. We first identify the structure of the optimal adversary, and through it we obtain the tight regret bound as $T \to \infty$. Before proceeding further, it is useful to recall that when the gains of the leading and lagging experts are given by $g+d$ and $g$, the MWA algorithm follows these experts with probabilities $\frac{e^{\eta d}}{e^{\eta d}+1}$ and $\frac{1}{e^{\eta d}+1}$ respectively. Thus, when the adversary increases $d$ by 1 i.e., increases the gain of the leading expert by $1$, the regret benchmark (namely, the gains of the leading expert) increases by $1$, where as MWA is correct only with probability $\frac{e^{\eta d}}{e^{\eta d}+1}$, and this therefore inflicts a regret of $\frac{1}{e^{\eta d}+1}$ on MWA. On the other hand, if the adversary decreases $d$ by 1, then the benchmark doesn't change, whereas MWA succeeds with probability $\frac{1}{e^{\eta d}+1}$, and this therefore inflicts a regret of $\frac{-1}{e^{\eta d}+1}$. When the adversary doesn't change $d$, the regret inflicted is $0$. 

\vspace{-0.4em}
\paragraph{Structure of the optimal adversary.} Let $\eta$ be the fixed update rate of the optimal MWA (the parameter in the exponent as explained in Section~\ref{sec:intro})\footnote{In fact, we can identify the optimal adversary for a much broader family of algorithms (see Appendix~\ref{app:two_conv_finite} for more details).}. Against a specific algorithm, an optimal adversary can always be found in the class of deterministic adversaries. The actions of the optimal adversary (against a specific MWA algorithm) depend only on the distance $d$ between leading and lagging experts and time step $t$. 
\begin{enumerate}
\itemsep-0.3em
\item \textbf{Loop aggregation:} At each time step, the adversary may either increase or decrease the gap $d$ by $1$, or leave $d$ unchanged. We denote these actions of the adversary by $d\overset{t}{\to} d+1$, $d\overset{t}{\to} d-1$, and $d\overset{t}{\to} d$. The respective regret values inflicted on the algorithm are given by  $\frac{1}{e^{\eta d}+1
}$, $\frac{-1}{e^{\eta d}+1}$, and $0$, which are all independent of the time when an action was taken. This means that if the adversary loops between $d$ and $d+1$ at several disconnected points of time, it may as well aggregate all of them and complete all of them in consecutive time steps. I.e., the optimal adversary starts at $d=0$ and then weakly monotonically increases $d$, stopping at various points $d=s$, looping for an arbitrary length of time between $d=s$ and $d=s-1$ and then proceeding forward. 
\item \textbf{Staying at same $d$ is dominated:} It is not hard to see that any action $x\rightarrow x$ is dominated for the adversary as this wastes a time step and inflicts $0$ regret on the algorithm. Thus the ``weakly monotonically increases'' in the previous paragraph can be replaced by ``strictly monotonically increases'' (except of course for the stopping points for looping). 
\item \textbf{Loop(0) domination:} Define $\Loop(d)\eqdef d\to d+1\to d$. It is easy to see that the regret inflicted by Loop($d$) is exactly $\frac{1}{e^{\eta d}+1}-\frac{1}{e^{\eta (d+1)}+1}$ and this quantity is maximized at $d=0$. Thus, the optimal adversary should replace all loops by loops at $0$. This gives us the structure claimed in Figure~\ref{fig:char} for the optimal adversary. 
\end{enumerate}

Given the optimal adversary's structure (as described in Figure~\ref{fig:char} 
) w.l.o.g. we can assume it to be looping for $\frac{T-\ell}{2}$ steps at $0$ 
and then monotonically increasing $d$ for $\ell$ steps at which point the game 
ends. 
In the following we will analyze the regret inflicted by the optimal adversary 
(which we showed was optimal for the class of algorithm $\algsingle$) against a 
broader class $\algdec$ of MWA. The regret of the adversary is: 

% \begin{align}
% \sum_{t=1}^{\frac{T-\ell}{2}}\left[\frac{1}{2}-\frac{1}{e^{\eta(2t)}+1}\right]+
% \sum_{d=1}^{\ell}\frac{1}{e^{(d-1)\cdot\eta(T-\ell+d) }+1}.
% \label{eq:mwa_regret_2_finite}
% \end{align}  

\begin{align}
\sum_{t=1}^{\frac{T-\ell}{2}}
\left[\frac{1}{2}-\frac{1}{e^{\eta(2t)}+1}\right]
+
\sum_{d=0}^{\ell-1}\frac{1}{e^{\eta(T-\ell+d+1)d}+1}.
\label{eq:mwa_regret_2_finite}
\end{align}  
%\vspace{-0.5em}
\paragraph{Asymptotic regret of the optimal adversary.}  We first notice that 
for a fixed adversary with a given $\ell$, the regret of MWA with decreasing 
$\eta(t)$ in \eqref{eq:mwa_regret_2_finite} is greater than or equal to the 
regret of MWA with a constant $\eta'=\eta(T-\ell)$, i.e., 
\begin{align}
\frac{T-\ell}{2}\left[\frac{1}{2}-\frac{1}{e^{\eta'}+1}\right]+
\sum_{d=0}^{\ell-1}\frac{1}{e^{d\eta'}+1}.
\label{eq:mwa_regret_2_finite_constant}
\end{align}  
This is true as each individual term in \eqref{eq:mwa_regret_2_finite_constant} is equal to or smaller than the corresponding term in \eqref{eq:mwa_regret_2_finite}. In the following we are going to use $\ell=T^{3/4}$ for the adversary and for convenience, we write $e^{\eta'(T-\ell)}=\tau=1+\frac{\alpha}{\sqrt{T}}$. The two terms in \eqref{eq:mwa_regret_2_finite} together place strong bounds on what $\alpha$ should be: they imply that $\alpha = \Theta(1)$.  We show this in $2$ steps: first we show that $\alpha = O(1)$, and then show that $\alpha = \Omega(1)$.
%we show that $\alpha > 0$, then
\begin{enumerate}[leftmargin=*]
\item The first term in \eqref{eq:mwa_regret_2_finite_constant} forces $\alpha$ to be $O(1)$. The regret of MWA for $\ell=T^{3/4}$ is at least
$$\frac{T-\ell}{2}\left[\frac{1}{2}-\frac{1}{e^{\eta'}+1}\right]\simeq\frac{T}{2}\left[\frac{1}{2}-\frac{1}{e^{\eta'}+1}\right] = \frac{\alpha\sqrt{T}}{4(1+e^{\eta'})}= 
\frac{\sqrt{T}}{4(\frac{2}{\alpha}+\frac{1}{\sqrt{T}})}.$$ 
Since MWA's regret upper bound in the finite horizon model is
$\Theta\left(\sqrt{T}\right)$, $\alpha$ must be $O(1)$. 
\item To show that $\alpha = \Theta(1)$ we argue that the regret from the second term of \eqref{eq:mwa_regret_2_finite_constant} is $\omega(\sqrt{T})$ when $\alpha = o(1)$. 
%In fact, we don't even need to consider $\alpha = o(1)$. Just assume that $\alpha$ is a very small constant $c$independent of $T$. Then $\tau=1+\frac{\alpha}{\sqrt{T}} = 1+\frac{c}{\sqrt{T}}$. 
For all $d \le \frac{\sqrt{T}}{\alpha}$, we have 
$\tau^d = (1+\frac{\alpha}{\sqrt{T}})^d \le e$. Thus MWA's regret for $k=\min(\frac{\sqrt{T}}{\alpha},T^{3/4})$ is at least
$$\sum_{d = 0}^{k-1} \frac{1}{\tau^d + 1} \ge \sum_{d = 0}^{k-1} \frac{1}{e+1} =\Omega(k)=\omega\left(\sqrt{T}\right).$$
Since MWA's regret upper bound in the finite horizon model is $\Theta\left(\sqrt{T}\right)$, we get $\alpha = \Omega(1)$.
\end{enumerate}

%We estimate $\eta'\sim e^\eta' - 1 = \frac{\alpha}{\sqrt{T}}$. 
%Now, given that $e^{\eta'}=1+\frac{\Theta(1)}{\sqrt{T}}$, we may assume without loss of generality that $\ell=o(T)$ in \eqref{eq:mwa_regret_2_finite}. 
Now, we obtain the following asymptotic estimate for the second part of \eqref{eq:mwa_regret_2_finite_constant}, where $\eta'\sim e^{\eta'} - 1 = \frac{\alpha}{\sqrt{T}}$. 

\begin{align}
\sum_{d=0}^{\ell-1}\frac{1}{\tau^d+1}\sim\int_{0}^{\ell}\frac{\mathrm{d}x}{e^{\eta' x}+1}
=\frac{1}{\eta'}\ln\left(2\frac{e^{\ell\eta'}}{e^{\ell\eta'}+1}\right)\sim \frac{\sqrt{T}}{\alpha}\left(\ln(2)-\ln(1+e^{-\ell\eta'})\right).
%=\sum_{d=0}^{\ell-1}\frac{(1-\delta)^d}{\tau^d+1}+
\label{eq:mwa_regret_2_finite_straight}
\end{align}

%$\frac{\sqrt{T}}{\alpha}\left(\ln(2)-\ln(1+e^{-\alpha\ell/\sqrt{T}})\right)$. 

The first part of \eqref{eq:mwa_regret_2_finite_constant} can be estimated as follows

\begin{equation}
\frac{T-\ell}{2}\left[\frac{1}{2}-\frac{1}{e^{\eta'}+1}\right]\sim\frac{T}{2}\left[\frac{e^{\eta'}-1}{2(e^{\eta'}+1)}\right]
\sim \frac{T}{2}\cdot\frac{{\eta'}}{4}=\frac{\alpha\sqrt{T}}{8}.
\label{eq:mwa_2_finite_second}
\end{equation}

%%the first term of \eqref{eq:mwa_regret_2_finite_constant}  
As $e^{-\ell\eta'}=e^{-\alpha T^{1/4}}=o(1)$, \eqref{eq:mwa_regret_2_finite_straight} simplifies to $\frac{\ln(2)\sqrt{T}}{\alpha}$, while 
the estimate for \eqref{eq:mwa_2_finite_second} is $\frac{\alpha\sqrt{T}}{8}$. Now the estimate $\sqrt{T}\left[\frac{\ln(2)}{\alpha} + \frac{\alpha}{8}\right]$ 
for the regret in \eqref{eq:mwa_regret_2_finite_constant} is minimized for the choice of parameter $\alpha=\sqrt{8\ln(2)}$. Then the regret of the optimal MWA is at least $\sqrt{T\cdot\frac{\ln(2)}{2}}(1+o(1))$. It is known that there is MWA for $k=2$ experts with regret at most $\sqrt{T\cdot\frac{\ln(2)}{2}}$ (asymptotic in $T$). 
Thus, we obtain the following claim~\ref{thm:2ExpertsLB_finite} (in the claim 
below, by ``optimal MWA'' we mean the MWA with the optimally tuned 
$\eta(t)=\eta'$). 
%\end{proof}
\begin{claim}
\label{thm:2ExpertsLB_finite}
For $\algfam = \algsingle,\algdec$: $\qquad R_2(\algfam,D_{lsdet}) = 
\sqrt{\frac{T\ln 
2}{2}} = R_2(\algfam,\advuniv)$.
%The optimal MWA algorithm for $2$ experts in the finite horizon model obtains 
%a regret of 
%$\sqrt{\frac{T\cdot\ln 2}{2}}$ as $T \to \infty$ among all algorithm in 
%$\algdec$.
\end{claim}
We generalize the adversary for $k=2$ and obtain a tight lower bound for 
$\algdec$ matching the known upper bound for arbitrary {\em even} number $k$ of 
experts and almost 
matching bound for {\em odd} number $k$ of experts. Since $R_k(\algsingle) \geq 
R_k(\algdec)$, the lower bound in Theorem~\ref{thm:kExpertsLB_finite} below 
applies to $\algsingle$ as well. 

\begin{theorem}
\label{thm:kExpertsLB_finite}
For $\algfam = \algsingle,\algdec$: $\qquad$ i) For even $k$: 
$R_k(\algfam,D_{lsdet}) = 
\sqrt{\frac{T\cdot\ln k}{2}} = 
R_k(\algfam,\advuniv)$.$\qquad$ 
ii) For odd $k$: $R_k(\algfam,D_{lsdet}) \geq \sqrt{\frac{T\cdot\ln 
k}{2}\left(1-\frac{1}{k^2}\right)} \geq 
\sqrt{1-\frac{1}{k^2}}R_k(\algfam,\advuniv)$.
%The asymptotically optimal (almost optimal for odd $k$) MWA from the class of 
%algorithms $\algdec$
%in the finite horizon model is the MWA with a constant $\eta(t)$. The regret 
%of 
%the optimal MWA, as $T\to\infty$ is
%(i) $\sqrt{\frac{T\cdot\ln k}{2}}$ for any {\em even} $k$, 
%\quad (ii) at least $\sqrt{\frac{T\cdot\ln k}{2}\left(1-\frac{1}{k^2}\right)}$ 
%for any {\em odd} $k$.
\end{theorem}
\begin{proof} Let $\eta(t)$ be the update rate of the optimal MWA, we define $\ell=T^{3/4}$ and $\eta'=\eta(T-\ell)$. We employ the following adversary for the even $k$ number of experts: 
\begin{enumerate}[leftmargin=*]
\itemsep-0.3em 
\item Divide all experts into two equal parties, numbered $A$ and $B$. For the first $\frac{T-\ell}{2}$ rounds ($\ell=T^{3/4}$), advance all the experts in party $A$ in even numbered rounds, and all experts in party $B$ in odd numbered rounds. 
\item For the remaining $\ell$ steps, pick an arbitrary expert and keep advancing just that expert.	
\end{enumerate}

Similar to \eqref{eq:mwa_regret_2_finite} this adversary obtains the regret of 
at least 
$\sum_{t=1}^{\frac{T-\ell}{2}}\left[\frac{1}{2}-\frac{1}{e^{\eta(2t)}+1}\right]
+\sum_{d=0}^{\ell-1}\frac{k-1}{e^{d\cdot\eta(T-\ell+d+1) }+k-1}$.
%\label{eq:mwa_regret_k_finite}.
%\vspace{-0.4em}
%\begin{align}
%\sum_{t=1}^{\frac{T-\ell}{2}}\left[\frac{1}{2}-\frac{1}{e^{\eta(2t)}+1}\right]
%+\sum_{d=0}^{\ell-1}\frac{k-1}{e^{d\cdot\eta(T-\ell+d+1) }+k-1}
%\label{eq:mwa_regret_k_finite}.
%\end{align}  
We further notice that similar to \eqref{eq:mwa_regret_2_finite_constant} the 
regret of MWA with decreasing $\eta(t)$ in the above expression is 
greater than or equal to the regret of MWA with a constant 
$\eta'=\eta(T-\ell)$, i.e., the previous expression is at least
\vspace{-0.4em}
\begin{align}
\frac{T-\ell}{2}\left[\frac{1}{2}-\frac{1}{e^{\eta'}+1}\right]
+\sum_{d=0}^{\ell-1}\frac{k-1}{e^{d\cdot\eta'}+k-1}.
\label{eq:mwa_regret_k_finite_constant}
\end{align}

% \begin{align}
% \frac{T-\ell}{2}\left[\frac{1}{2}-\frac{1}{e^{\eta'}+1}\right]+
% \sum_{d=0}^{\ell-1}\frac{k-1}{e^{\eta' d}+k-1}.
% \label{eq:mwa_regret_k_finite}
% \end{align}  
\vspace{-0.4em}
We use \eqref{eq:mwa_2_finite_second} to estimate the first term of \eqref{eq:mwa_regret_k_finite_constant}.
%The second term of \eqref{eq:mwa_regret_k_finite_constant} can also be estimated similar to \eqref{eq:mwa_regret_2_finite_straight} as follows.
We estimate the second term of \eqref{eq:mwa_regret_k_finite_constant} similar to \eqref{eq:mwa_regret_2_finite_straight} as follows.
\vspace{-0.4em}
\begin{align}
\sum_{d=0}^{\ell-1}\frac{k-1}{e^{d\eta'}+k-1}\sim\int_{0}^{\ell}\frac{(k-1)\mathrm{d}x}{e^{x\eta'}+k-1}
=\frac{1}{\eta'}\ln\left(\frac{k\cdot e^{\ell\eta'}}{e^{\ell\eta'}+k-1}\right)\sim 
\frac{\sqrt{T}\ln(k)}{\alpha}.
\label{eq:mwa_regret_k_finite_straight}
\end{align}
\vspace{-0.4em}
Now, combining these two estimates the regret from \eqref{eq:mwa_regret_k_finite_constant} is at least 
\vspace{-0.4em}
\[
\frac{\alpha\sqrt{T}}{8} + \frac{\sqrt{T}\ln(k)}{\alpha}
\ge 2\cdot \sqrt{\frac{\alpha\sqrt{T}}{8} \cdot \frac{\sqrt{T}\ln(k)}{\alpha}} = \sqrt{\frac{T\ln(k)}{2}},
\]
which precisely matches the upper bound on the regret of MWA\cite{BFHHSW97}.

For the odd $k$ number of experts we  employ almost the same adversary as for even $k$, although, since $k$ now is 
odd, we split experts into two parties of {\em almost} equal sizes (see Appendix~\ref{app:mwa_odd_finite} for full details).
\end{proof}

\subsection{General variations of MWA}
\label{sec:varMWA}
We have seen that the best known MWA with a flat learning rate $\eta$ achieves 
optimal (or almost optimal in the case of odd number of experts) regret among 
all MWAs with monotone decreasing learning rates $\eta(t)$. However, it seems 
that in the finite horizon model a better strategy for tuning parameters of MWA 
would be to use higher rates $\eta(t)$ towards the end $T$. In the following we 
study a broader family of MW algorithms $\algarb$ where learning parameter 
$\eta(t)$ can vary in an arbitrary way. 
In the following theorem we show that such adaptivity of MWA cannot decrease the regret of the algorithm by more than a factor of $2/3$.

\begin{remark} In fact, our analysis extends to the family $\algrand$ where 
each $\eta(t)$ can be a random variable drawn from a distribution $F_t$. 
Effectively, with a random $\eta(t)$ the algorithm player can get any convex 
combination $f(\Gainit[i][t-1],t)=\Ex[\eta(t)]{e^{\eta(t)\Gainit[i][t-1]}}$ of 
$e^{\eta(t)\Gainit[i][t-1]}$ in the vector of probabilities for following each 
expert $i$ at time $t$. 
This constitutes a much richer family of algorithms compared to the 
standard single parameter MWA family.   
\end{remark}

\begin{theorem}
\label{thm:algrand}
For $\algfam = \algarb,\algrand$:$\qquad$
i) For even $k$: $R_k(\algfam,D_{lsrand}) \geq  
\frac{2}{3}\sqrt{\frac{T\cdot\ln k}{2}} \geq \frac{2}{3}R_k(\algfam,\advuniv)$.
ii) For odd $k$: $R_k(\algfam,D_{lsrand}) \geq  
\frac{2}{3}\sqrt{\frac{T\cdot\ln k}{2}\left(1-\frac{1}{k^2}\right)} \geq 
\frac{2}{3}\sqrt{1-\frac{1}{k^2}}R_k(\algfam,\advuniv)$.
%\label{thm:kExpertsLB_var}
\end{theorem}

\begin{proof}
Define $\ell=T^{3/4}$ and $R=\left[\frac{T-\ell-1}{2}\right]$. We use the 
following adversary for even number of experts $k$: 

\begin{enumerate}
\itemsep-0.3em
\item Choose $j\in [R]$ uniformly at random. With probability $0.5$ don't advance any expert in the first step.
\item Divide all experts into two equal parties, numbered $A$ and $B$.  For the next $j$ rounds, advance all the experts in party $A$ in even numbered rounds, and all experts in party $B$ in odd numbered rounds. 
\item For next $\ell$ steps, pick any expert $i$ and keep advancing 
just expert $i$. Do nothing in remaining steps.
\end{enumerate} 
\vspace{-0.4em}
The regret of the algorithm  is 
\vspace{-0.4em}
\begin{align}
\frac{1}{R}\sum_{j=0}^{R-1}\bigg[& 
\frac{1}{2}\left(\sum_{t=1}^{j}\left[\frac{1}{2}-\frac{1}{e^{\eta(2t)}+1}\right]+\sum_{d=0}^{\ell-1}\frac{k-1}{e^{d\cdot\eta(2j+d+1)}+k-1}\right)+\nonumber\\
& 
\frac{1}{2}\left(\sum_{t=1}^{j}\left[\frac{1}{2}-\frac{1}{e^{\eta(2t+1)}+1}\right]+\sum_{d=0}^{\ell-1}\frac{k-1}{e^{d\cdot\eta(2j+d+2)}+k-1}\right)
\bigg].
\label{eq:mwa_regret_k_var_complex} 
\end{align}
                                                                                                                                                                                                                                                                                                                                                                                                                                                                                                                                                                                                                                                                                                                                                                                                                                                             
Since $\eta(t)$ can be arbitrary nonnegative number, we break \eqref{eq:mwa_regret_k_var_complex} into
terms with the same $\eta(t)$ (we also drop a few terms to simplify the expression). 
In the following, we will also assume that $e^{\eta(t)}=1+\frac{\alpha(t)}{\sqrt{T}}$, where 
$\alpha(t) = \Theta(1)$ for every $t\in [T]$. Later we will explain why this assumption is without loss of generality. 
\vspace{-0.4em}
\begin{align} 
\eqref{eq:mwa_regret_k_var_complex} &\ge \frac{1}{2R}\sum_{t=\ell}^{T-\ell-1}\left(\left[R-\lceil
t/2\rceil\right]\cdot\left[\frac{1}{2}-\frac{1}{e^{\eta(t)}+1}\right]
+\sum_{d=0}^{\ell-1}\frac{k-1}{e^{\eta(t)d}+k-1}\right)\nonumber\\
&\simeq \sum_{t=\ell}^{T-\ell-1}\left(\left[\frac{R-\lceil t/2\rceil}{R}\right]\frac{\alpha(t)}{8\sqrt{T}}+\frac{\sqrt{T}\ln(k)}{2R\cdot\alpha(t)}\right)
\ge \sum_{t=\ell}^{T-\ell-1}2\sqrt{\left[\frac{R-\lceil t/2\rceil}{R}\right]\frac{\alpha(t)}{8\sqrt{T}}
\frac{\sqrt{T}\ln(k)}{2R\cdot\alpha(t)}}\nonumber\\
& =\sqrt{\frac{\ln(k)}{2\cdot 2R}}\sum_{t=\ell}^{T-\ell-1}\sqrt{\left[\frac{R-\lceil t/2\rceil}{R}\right]}
\simeq \sqrt{\frac{\ln(k)}{2\cdot T}}\int_{0}^{1}\sqrt{1-x}\quad\mathrm{d}x=\sqrt{\frac{\ln(k)}{2T}}\cdot \frac{2}{3}.
\label{eq:mwa_regret_k_var_derivation}
\end{align}
In the above derivation we obtain the first approximation $\simeq$ by using approximations from \eqref{eq:mwa_2_finite_second} and \eqref{eq:mwa_regret_k_finite_straight}. 

We now argue that the assumption $\alpha(t) = \Theta(1)$ is without loss of generality for every 
$t\in [T]$. We apply a similar argument as in Theorem~\ref{thm:2ExpertsLB_finite}, but now for each 
individual term with a particular $\eta(t)$. The term 
$\sum_{d=0}^{\ell-1}\frac{k-1}{e^{\eta(t)d}+k-1}$ in \eqref{eq:mwa_regret_k_var_derivation} is 
already large enough for the estimate when $\alpha(t)=o(1)$. The term 
$\left[R-\lceil t/2\rceil\right]\cdot\left[\frac{1}{2}-\frac{1}{e^{\eta(t)}+1}\right]$ also places a 
strong bound of $O(1)$ on $\alpha(t)$, when $\left[R-\lceil t/2\rceil\right]$ is constant fraction of 
$T$. To argue about $t$ close to the threshold $T$, we can slightly modify the adversary by playing with 
a small constant probability $\eps$ entirely ``looping'' strategy (without ``straight line'' part). This 
would make the coefficient in front of $\left[\frac{1}{2}-\frac{1}{e^{\eta(t)}+1}\right]$ to be 
sufficiently large, and at the same time would decrease the lower bound by at most $1-\eps$ factor. 
Taking $\eps$ arbitrary small we obtain the bound in \eqref{eq:mwa_regret_k_var_derivation}. This concludes the proof for the even number of experts

For the odd number of experts $k$. We slightly modify the adversary analogous to the case of odd number of experts in Theorem~\ref{thm:kExpertsLB_finite}. This gives us an additional factor of $1-\frac{1}{k^2}$ for each of the looping terms.
\end{proof}
\vspace{-0.6em}
\begin{remark}
One can slightly improve the lower bound in Theorem~\ref{thm:algrand} 
and get a better factor than $\frac{2}{3}$. To this end we employ a more 
complicated adversary by playing with some probability $p>0$ the same strategy 
as in Theorem~\ref{thm:algrand} and with the remaining $1-p$ probability playing
purely looping strategy (see Appendix~\ref{app:mwa_gen}).
\end{remark}

\section{Geometric horizon}
\label{sec:delta}
We prove two main results in this Section. We derive the structure of the 
optimal adversary for $2$ experts and show that the optimal regret for $2$ 
experts is exactly $\frac{0.391}{\sqrt{\delta}}$ as $\delta \to 0$ (see 
Appendix~\ref{app:mwa_delta_two}). For an arbitrary number of experts $k$, we 
derive a regret lower bound of $\frac{1}{2}\sqrt{\frac{\ln(k)}{2\delta}}$ (see 
Appendix~\ref{app:mwa}).

\bibliographystyle{plainnat}
\bibliography{bibs,machine_learning}

\appendix

\section{Finite horizon}
\subsection{Theorem~\ref{thm:kExpertsLB_finite} for odd number of experts $k$}
\label{app:mwa_odd_finite}
\begin{oneshot}{Theorem~\ref{thm:kExpertsLB_finite}} 
For $\algfam = \algsingle,\algdec$: $\qquad$ i) For even $k$: 
$R_k(\algfam,D_{lsdet}) = 
\sqrt{\frac{T\cdot\ln k}{2}} = 
R_k(\algfam,\advuniv)$.$\qquad$ 
ii) For odd $k$: $R_k(\algfam,D_{lsdet}) \geq \sqrt{\frac{T\cdot\ln 
k}{2}\left(1-\frac{1}{k^2}\right)} \geq 
\sqrt{1-\frac{1}{k^2}}R_k(\algfam,\advuniv)$.
\end{oneshot}
\begin{proof}
We have already proven this theorem for even $k$. So let $k=2\cdot m + 1$ and 
let $\eta(t)$ be as before the update rate (non increasing in $t$) of the 
optimal MWA. We employ almost the same adversary as for even $k$, although, 
since $k$ now is odd, we split experts into two parties of {\em almost} equal 
sizes. 
As in the case of even $k$ we choose $\ell=T^{3/4}$ and let $\eta'=\eta(T-\ell)$. 
\begin{enumerate}[leftmargin=*]
\itemsep-0.3em 
\item divide all experts into two parties of sizes $m$ and $m+1$ respectively, advance all experts within the party of size $m$ in one step then advance all
			experts	within the other party of size $m+1$ in the next step, repeat these cycles of two steps $\frac{T-\ell}{2}$ times.    
\item fix on one expert and keep advancing just that expert for the remaining $\ell$ steps.	
\end{enumerate}

This adversary obtains the regret of at least

\begin{align}
\frac{T-\ell}{2}\left[\frac{m+1}{2m+1}-\frac{m+1}{m\cdot e^{\eta'}+m+1}\right]+
\sum_{d=0}^{\ell-1}\frac{k-1}{e^{d\eta'}+k-1}.
\label{eq:mwa_regret_k_finite_odd}
\end{align}  

 We use the estimate \eqref{eq:mwa_regret_k_finite_straight}
for the second part of \eqref{eq:mwa_regret_k_finite_odd}. For the first part of \eqref{eq:mwa_regret_k_finite_odd} 
we closely follow the derivation in \eqref{eq:mwa_2_finite_second} and obtain the following estimate.

\begin{equation}
\frac{T-\ell}{2}\left[\frac{m+1}{2m+1}-\frac{m+1}{m\cdot e^{\eta'}+m+1}\right]
\sim\frac{T}{2}\cdot\frac{(m+1)\cdot m\cdot\alpha/\sqrt{T}}{(2m+1)^2}
=\frac{\alpha\sqrt{T}(k-1)(k+1)}{8 k^2}.
\label{eq:mwa_k_finite_second_odd}
\end{equation}

Therefore, the regret from \eqref{eq:mwa_regret_k_finite_odd} is at least 

\[
\frac{\alpha\sqrt{T}}{8}\cdot\left(1-\frac{1}{k^2}\right) + \frac{\sqrt{T}\ln(k)}{\alpha}
\ge 2\cdot \sqrt{\frac{\alpha\sqrt{T}}{8}\left(1-\frac{1}{k^2}\right) \cdot \frac{\sqrt{T}\ln(k)}{\alpha}} = \sqrt{\frac{T\ln(k)}{2}\left(1-\frac{1}{k^2}\right)},
\]
which concludes the proof of the theorem.
\end{proof}
\subsection{Improved lower bound for $\algrand$.}
\label{app:mwa_gen}
We use a slightly more complicated adversary than the one used in 
Theorem~\ref{thm:algrand} to get a better regret. Call this adversary (which we 
describe in the proof of the theorem below) 
$D_{lsrand++}$. We show:  
\begin{theorem}
\label{thm:kExpertsLB_var}
For $\algfam = \algarb,\algrand$:$\qquad$
i) For even $k$: $R_k(\algfam,D_{lsrand++}) \geq  
0.68\sqrt{\frac{T\cdot\ln k}{2}} \geq 0.68 R_k(\algfam,\advuniv)$.
ii) For odd $k$: $R_k(\algfam,D_{lsrand++}) \geq  
0.68 \sqrt{\frac{T\cdot\ln k}{2}\left(1-\frac{1}{k^2}\right)} \geq 
0.68 \sqrt{1-\frac{1}{k^2}}R_k(\algfam,\advuniv)$.
\end{theorem}

\begin{proof}
We closely follow the proof of Theorem~\ref{thm:algrand}, although now we 
employ a more complicated adversary by following with a probability $p>0$ the 
same strategy as in Theorem~\ref{thm:algrand} and with the remaining 
$1-p$ probability playing purely looping strategy (without ``straight line'' 
part) as follows.
\begin{enumerate}
\item With probability $0.5$ don't advance any expert in the first step.
\item Divide all experts into two equal (almost equal, when $k$ is odd ) parties, numbered $A$ and $B$.  For the next $T$ rounds, advance all the experts in party $A$ in even numbered rounds, and all experts in party $B$ in odd numbered rounds. 
\end{enumerate} 

The regret of the MWA for even $k$ with respect to this adversary is
\begin{equation}
p\cdot\eqref{eq:mwa_regret_k_var_complex} + (1-p)\cdot\left(
\frac{1}{2}\sum_{t=1}^{\left\lfloor\frac{T}{2}\right\rfloor}\left[\frac{1}{2}-\frac{1}{e^{\eta(2t)}+1}\right]+\frac{1}{2}\sum_{t=1}^{\left\lfloor\frac{T-1}{2}\right\rfloor}\left[\frac{1}{2}-\frac{1}{e^{\eta(2t+1)}+1}\right]\right)
\label{eq:mwa_regret_k_var_complex_looping}
\end{equation}
We closely follow derivation \eqref{eq:mwa_regret_k_var_derivation} and assume
that $e^{\eta(t)}=1+\frac{\alpha(t)}{\sqrt{T}}$, where $\alpha(t) = \Theta(1)$ for every $t\in [T]$.
We also recall that $R=\left[\frac{T-\ell-1}{2}\right]$.
                                                                                                                                                                                                                                                                                                                                                                                                                                                                                                                                                                                                                                                                                                                                                                                                                                                          
\begin{align} 
\eqref{eq:mwa_regret_k_var_complex_looping} &\gtrsim 
\sum_{t=\ell}^{T-\ell-1}\left(\left[p\cdot\frac{R-\lceil t/2\rceil}{R}+1-p\right]\frac{\alpha(t)}{8\sqrt{T}}+p\cdot\frac{\sqrt{T}\ln(k)}{2R\cdot\alpha(t)}\right)
\nonumber\\
&\ge \sum_{t=\ell}^{T-\ell-1}2\sqrt{\left[p\cdot\frac{R-\lceil t/2\rceil}{R}+1-p\right]\frac{\alpha(t)}{8\sqrt{T}}\cdot
p\cdot\frac{\sqrt{T}\ln(k)}{2R\cdot\alpha(t)}}\nonumber\\
& \simeq \sqrt{\frac{p\ln(k)}{2\cdot T}}\int_{0}^{1}\sqrt{p(1-x)+1-p}\quad\mathrm{d}x=
\sqrt{\frac{\ln(k)}{2T}}\cdot \frac{2}{3\sqrt{p}}\left(1-(1-p)^{3/2}\right).
\label{eq:mwa_regret_k_var_derivation_looping}
\end{align}
The right hand side of expression \eqref{eq:mwa_regret_k_var_derivation_looping} is maximized for
$p\sim 0.866$ at a value slightly larger than $0.68$. Note that we could assume that $\alpha(t) = \Theta(1)$, as otherwise either the ``looping'' term, or the ``straight-line'' respective term for the particular $\eta(t)$ would be greater than the bound we used in \eqref{eq:mwa_regret_k_var_derivation_looping}.

The derivation for the odd number of experts $k$ is almost the same with an additional factor of $1-\frac{1}{k^2}$ for each of the looping terms.
\end{proof}
\subsection{Optimal adversary for two experts for a broader family than MWA.}
\label{app:two_conv_finite}
Our goal in this section is to identify the structure of the optimal adversary for a broader family of algorithms than MWA. We consider a simple case of $k=2$ experts. We assume that in general any algorithm  is parametrized by the distance $d$ between lagging and leading experts at time $t$ and picks the lagging expert with probability $p(d,t)$ and the leading expert with probability $1-p(d,t)$. Thus, when the adversary increases $d$ by 1, i.e., increases the gain of the leading expert by $1$, the regret benchmark (namely, the gains of the leading expert) increases by $1$, where as the algorithm is correct only with probability $1-p(d,t)$, and this therefore inflicts a regret of $p(d,t)$ on the algorithm. On the other hand, if the adversary decreases $d$ by 1, then the benchmark doesn't change, whereas  the algorithm succeeds with probability $p(d,t)$, and this therefore inflicts a regret of $-p(d,t)$. When the adversary doesn't change $d$, the regret inflicted is $0$. For $k=2$ experts we consider a family of algorithms $\algconv$ given by the following two properties:
\begin{itemize}
\item[(i)]  $p(d,t)$ does not decrease with $t$ for a fixed $d$; 
\item[(ii)] $p(d,t)$ is convex and decreasing in variable $d$ for a fixed $t$, i.e., 
            $p(d+1,t)+p(d-1,t)\ge 2p(d,t)$ and $p(d,t)\le p(d-1,t)$ for any $d\ge 1$.
\end{itemize}
Note that family $\algconv$ contains $\algdec$ for $k=2$ experts.
 
Against a specific algorithm, an optimal adversary can always be found in the class of deterministic adversaries. The actions of the optimal adversary (against a specific algorithm) depend only on the distance $d$ between leading and lagging experts and time step $t$. At each time step, the adversary may either increase or decrease the gap $d$ by $1$, or leave $d$ unchanged. We denote these actions of the adversary by $d\overset{t}{\to} d+1$, $d\overset{t}{\to} d-1$, and $d\overset{t}{\to} d$. 
\begin{enumerate}
\item \textbf{Pushing $\mathbf{d\to d}$ to $\mathbf{d=0}$}: we observe that at any time $t$ two consecutive actions $d\overset{t}{\to} d+1$ and $d+1\overset{t+1}{\to} d+1$ give at most as much regret as $d\overset{t}{\to} d$ and $d\overset{t+1}{\to} d+1$, since $p(d,t)+0\le 0 +p(d,t+1)$. This means that the optimal adversary can always use the latter pair of actions instead of the former pair of actions (the rest of the actions are unaffected) and inflict at least as much regret on the algorithm.
\item \textbf{Pushing loops $\mathbf{d\to d+1\to d}$ to $\mathbf{0\to 1\to 0}$}:
Let $t_0+1$ be the first time the optimal adversary has decreased distance $d$ between lagging and leading experts, i.e., the adversary played $d+1\overset{t_0+1}{\to} d$. We may assume by the previous observation that the adversary plays the incremental actions only at the end of the $[t_0]$ interval, i.e., he first plays $0\overset{t}{\to} 0$ for $1\le t<t_0-d$ times and then plays $0\overset{t_0-d}{\to}1\overset{t_0-d+1}{\to} 2\overset{t_0-d+2}{\to}\dots\overset{t_0}{\to}d+1\overset{t_0+1}{\to}d$. Suppose that $d>0$, then the adversary would not decrease the regret by substituting actions 
$d\overset{t_0}{\to}d+1\overset{t_0+1}{\to}d$ with $d\overset{t_0}{\to}d-1\overset{t_0+1}{\to}d$. Indeed,
using properties (ii) and (i) for the algorithm we get
\[
p(d-1,t_0+1)+ p(d+1,t_0+1) \ge 2p(d,t_0+1)\ge 2p(d,t_0).  
\]
Equivalently, $\text{Regret}(d\overset{t_0}{\to}d+1\overset{t_0+1}{\to}d)=p(d,t_0)-p(d+1,t_0+1)\le -p(d,t_0)+p(d-1,t_0+1)=\text{Regret}(
d\overset{t_0}{\to}d-1\overset{t_0+1}{\to}d)$. Furthermore, the play of adversary is dominated by
$0\overset{t_0-d}{\to}1\overset{t_0-d+1}{\to}0\overset{t_0-d+2}{\to}1\overset{t_0-d+3}{\to}2\dots\overset{t_0}{\to}d$. The same argument applied to the time $t$ after $t_0-d+1$ allows us 
to say that all actions $d\to d-1$ (together with some previous $d-1\to d$ action) are dominated by the loops $0\to 1\to 0$. To this end, we can conclude that 
adversary play is dominated by a sequence: $0\to 0\to\dots\to 0\to 1\to 0\to\dots\to 0\to\dots 0
\to 1\to 2\to\dots\to d$.
\item \textbf{$\mathbf{0\to 0}$ is dominated}: We observe that at any time $t$ the sequence $0\overset{t}{\to}1\overset{t+1}{\to}0
\overset{t+2}{\to}0$ dominates $0\overset{t}{\to}0\overset{t+1}{\to}1\overset{t+2}{\to}0$. Indeed, 
\[
p(0,t)-p(1,t+1) = 0.5 -p(1,t+1) \ge 0.5 -p(1,t+2) = p(0,t+1)-p(1,t+2). 
\]
Thus, adversary plays all sequences $0\to 1\to 0$ before $0\to 0$. Finally, two consecutive actions 
$0\overset{t}{\to}0$, $0\overset{t+1}{\to}0$ are dominated by $0\overset{t}{\to}1\overset{t+1}{\to}0$, so
the optimal adversary would use at most one action of the form $0\to 0$.
\end{enumerate}

Without loss of generality, the optimal adversary can be assumed to be looping 
for $\frac{T-\ell}{2}$ steps at $0$ ($0\to 1\to 0$) and then monotonically 
increasing $d$ for $\ell$ steps at which point the game ends\footnote{The 
optimal adversary could also use one extra single step loop $0\to 0$, which is 
not important for our asymptotic analysis as $T$ goes to infinity.}. 
\section{Geometric horizon}
\subsection{Asymptotic regret of the optimal adversary for $2$ experts.}
\label{app:mwa_delta_two}

We begin our analysis of MWA with the case of $2$ experts and first identify 
the structure of the optimal adversary. 
This adversary turns out to be in some sense the opposite to the optimal 
adversary in the finite horizon model: it 
keeps increasing the gap between lagging and leading experts for the first few 
rounds and only then enters the looping
phase until the process stops --- we denote this by $\mathcal{D}_{sl}$.  
\paragraph{Structure of the optimal adversary.} 
%Let $\eta$ be the update rate of the optimal MWA. 
In the geometric horizon model the actions of the optimal adversary against any 
specific algorithm depend only on the distance $d$ between leading and lagging 
experts and don't depend on the time step. Thus to describe the optimal 
adversary, one just needs to specify what will the lag in the next step be, 
given that it is $d$ in this step.  Further, the optimal adversary at no value 
of $d$ will decide to maintain the same $d$ in next step, i.e., it advances 
exactly one expert at a time, thereby increasing $d$ by $1$ or decreasing $d$ 
by $1$ in each step. Without loss of generality we assume that the optimal 
adversary advances the leading expert for the first $\ell+1$ steps ($\ell$ may 
be infinite) and at step $\ell+1$ advances the lagging expert. At that point 
$d=\ell$, and since the optimal adversary advanced the leading expert the first 
time $d$ was equal to $\ell$, it will do the same now too. Extending this 
reasoning to all remaining steps we conclude that the optimal adversary follows 
straight line strategy for the first $\ell$ steps and at that point when $d = 
\ell$, switches to the looping strategy for the remaining time with $d$ looping 
between $\ell$ and $\ell+1$. 
\begin{remark}
\label{rem:geom_2_general}
Note that the above structural reasoning is applicable not just for the family 
$\algsingle$, but for the entire universe of algorithms $\alguniv$. While the 
following precise calculations on regret are for the family $\algsingle$ of 
MWA, the very simple adversary structure is fully general and applies to 
$\alguniv$ also, i.e., $R_2(\alguniv,\mathcal{D}_{sl}) = R_2(\alguniv, 
\advuniv)$.
\end{remark}

For a particular MWA with a parameter $\eta$ the optimal adversary achieves the 
regret of

\begin{align}
\label{eq:mwa_regret_2}
\sum_{d=0}^{\ell-1}\frac{(1-\delta)^{d+1}}{e^{\eta 
d}+1}+(1-\delta)^{\ell+1}\sum_{k=0}^{\infty}\left[\frac{(1-\delta)^{2k}}{e^{\eta\ell}+1}
-\frac{(1-\delta)^{2k+1}}{e^{\eta(\ell+1)}+1}\right].
%=\sum_{d=0}^{\ell-1}\frac{(1-\delta)^d}{\tau^d+1}+
\end{align}

We analyze the asymptotic regret of the optimal adversary given in the 
above expression and compute the regret of the optimal algorithm in case of 
$k=2$ experts for the family $\algsingle$ in Theorem~\ref{thm:2ExpertsLB} 
below. 
\begin{remark}
\label{rem:GeomVsFinite}
In contrast to the finite horizon model, the analysis of looping and 
straight-line
strategies of the optimal adversary in geometric horizon are related and cannot 
be decomposed into 
two independent quantities. This is because the looping phase comes after the 
straight-line phase in the geometric horizon model, and therefore the regret it 
inflicts 
is strongly influenced by the 
number of straight-line steps that have passed. On the other hand, in the 
finite horizon model, the straight-line phase follows the looping phase, and 
the regret inflicted by the straight-line phase is almost independent of the 
number of loops that have passed (the only dependence is via the number of 
rounds in the game that remain and it does not matter for asymptotics).  This 
fundamental 
difference in structure manifests in how the optimal regret values compare in 
these two settings: where as in finite horizon we have a regret value of 
$\sqrt{\frac{T\ln 2}{2}}$, in the geometric horizon setting instead of having 
the equivalent $\sqrt{\frac{\ln 2}{2\delta}}$, we get the optimal regret to be 
$\frac{0.391}{\sqrt{\delta}}$ --- the former is 50.5\% larger than the latter. 
In other words, it is more difficult for the adversary to inflict regret in the 
geometric horizon setting. 
\end{remark}

\begin{theorem}
\label{thm:2ExpertsLB}
%The best MWA for $2$ experts obtains a regret of $\frac{0.391}{\sqrt{\delta}}$ 
%as $\delta \to 0$ in the geometric horizon model. 
In geometric horizon, $R_2(\algsingle,\mathcal{D}_{sl}) = 
\frac{0.391}{\sqrt{\delta}} = R_2(\algsingle,\advuniv)$. 
\end{theorem}
\begin{proof}
We denote $e^{\eta}=\tau=1+\alpha\sqrt{\delta}$. We also can estimate $\eta\sim 
e^\eta - 1 = \alpha\sqrt{\delta}$, where $\alpha=\Omega(1)$ (if $\alpha\neq
\Omega(1)$, then, similar to the finite horizon model, straight or looping 
adversary alone already achieves regret of $\omega(\frac{1}{\sqrt\delta})$).
Without loss of generality we assume that 
$\ell=\beta\cdot\frac{1}{\sqrt{\delta}}=o(\frac{1}{\delta}).$ Then 
$(1-\delta)^\ell\sim 1$.
We further analyze separately each part of \eqref{eq:mwa_regret_2}. First, 
%similar to \eqref{eq:mwa_k_experts_straight} 
we have

\begin{align}
\sum_{d=0}^{\ell-1}\frac{(1-\delta)^{d+1}}{\tau^d+1}  & \sim  
\sum_{d=0}^{\ell-1}\frac{1}{\tau^d+1}\sim 
\int_{0}^{\ell}\frac{\mathrm{d}x}{e^{\eta x}+1} \nonumber \\
& =\frac{1}{\eta}\ln\left(2\frac{e^{\ell\eta}}{e^{\ell\eta}+1}\right)\sim 
\frac{\ln(2)}{\alpha\sqrt{\delta}}-\frac{\ln(1+e^{-\alpha\beta})}{\alpha\sqrt{\delta}}.
%=\sum_{d=0}^{\ell-1}\frac{(1-\delta)^d}{\tau^d+1}+
\label{eq:simple_path_regret}
\end{align}

The second part we can estimate as follows
\begin{align*}
(1-\delta)^{\ell+1}\sum_{k=0}^{\infty}\left[\frac{(1-\delta)^{2k}}{\tau^\ell+1}
-\frac{(1-\delta)^{2k+1}}{\tau^{\ell+1}+1}\right]\sim
\left[\frac{1}{\tau^\ell+1}-\frac{1-\delta}{\tau^{\ell+1}+1}\right]\cdot\sum_{k=0}^{\infty}(1-\delta)^{2k}\\
=\left[\frac{1}{\tau^\ell+1}-\frac{1}{\tau^{\ell+1}+1}+\frac{\delta}{\tau^{\ell+1}+1}\right]\frac{1}{2\delta-\delta^2}\sim
O(1)+\left[\frac{\tau^\ell(\tau-1)}{(\tau^\ell+1)(\tau^{\ell+1}+1)}\right]\frac{1}{2\delta}\\
\sim 
\frac{1}{2\delta}\cdot\frac{e^{\alpha\beta}\cdot\alpha\sqrt{\delta}}{(1+e^{\alpha\beta})^2}.
\end{align*}

Combining these two estimates we get that \eqref{eq:mwa_regret_2} is 
asymptotically equal to

\begin{align}
\frac{\ln(2)}{\alpha\sqrt{\delta}}-\frac{\ln(1+e^{-\alpha\beta})}{\alpha\sqrt{\delta}}+
\frac{1}{2\delta}\cdot\frac{e^{\alpha\beta}\cdot\alpha\sqrt{\delta}}{(1+e^{\alpha\beta})^2}=
\frac{\ln(2)}{\sqrt{\delta}\alpha}-\frac{\ln(1+e^{-\alpha\beta})}{\sqrt{\delta}\alpha}+\frac{\alpha}{2\sqrt{\delta}}\frac{e^{-\alpha\beta}}{(1+e^{-\alpha\beta})^2}.
\label{eq:optimization_alpha_beta}
\end{align}

With a new notation $\gamma=\frac{1}{1+e^{-\alpha\beta}}$ the above expression 
becomes
\begin{align*}
\frac{\ln(2)}{\alpha\sqrt{\delta}}+\frac{\ln(\gamma)}{\alpha\sqrt{\delta}}+\frac{\alpha}{2\sqrt{\delta}}\gamma(1-\gamma)=
\frac{\ln(2\gamma)}{\alpha\sqrt{\delta}}+\frac{\alpha}{2\sqrt{\delta}}\gamma(1-\gamma)=
\frac{1}{\sqrt{\delta}}\cdot h(\alpha,\gamma).
\end{align*}
We recall that the player first picks parameter $\alpha\in(0,\infty)$ for the 
algorithm and then the adversary decides on the optimal
$\gamma(\alpha)=\frac{1}{1+e^{-\alpha\beta}}$ which can be anything in 
$[\frac{1}{2},1]$. In other words we
are looking for
\begin{equation}
\label{eq:min_max_alpha_gamma}
\inf_{\alpha\in(0,\infty)}\max_{\gamma(\alpha)\in[\frac{1}{2},1]} 
h(\alpha,\gamma).
\end{equation}
For each fixed $\alpha$ we conclude that $\gamma(\alpha)$ is either 
$\frac{1}{2}$, or $1$, or such that 
$\frac{\mathrm{d}}{\mathrm{d}\gamma}h(\gamma)=\frac{1}{\alpha\gamma}+\frac{\alpha}{2}-\alpha\gamma=0$.
 The latter expression gives us a quadratic equation on $\gamma$, which has 
only one positive root:
\[
\gamma(\alpha)=\frac{\frac{1}{2}+\sqrt{\frac{1}{4}+\frac{4}{\alpha^2}}}{2}.
\]
We note that $h_\gamma(\alpha,\gamma)>0$, when $\gamma=\frac{1}{2}$. Therefore, 
$h(\alpha,\gamma)$ does not attain its maximum at $\gamma=\frac{1}{2}$.
For $\alpha<\sqrt{2}$ the above expression is greater than $1$ and, hence, 
$\gamma(\alpha)=1$. We further note that
$h_\gamma(\alpha,\gamma)=\frac{1}{\alpha}-\frac{\alpha}{2} < 0$ at $\gamma=1$. 
Thus 
$\gamma(\alpha)=\frac{\frac{1}{2}+\sqrt{\frac{1}{4}+\frac{4}{\alpha^2}}}{2}$
is a unique maximum of $h(\alpha,\gamma)$ for any $\alpha>\sqrt{2}$.

Now we need to find optimal $\alpha$ in
\[
\inf_{\alpha\in(0,\infty)}h(\alpha,\gamma(\alpha)).
\]
When $\alpha\in[0,\frac{1}{2}]$ we have $\gamma(\alpha)=1$ and 
$h(\alpha,\gamma)=\frac{\ln(2)}{\alpha}$, which attains its minimum of 
$\frac{\ln(2)}{\sqrt{2}}$ at $\alpha=\sqrt{2}$.

We know that the optimal $\alpha=\Theta(1)$.
Thus $h(\alpha,\gamma(\alpha))$ attains its minimum for some finite $\alpha$. 
For $\alpha\in[\sqrt{2},\infty)$ this minimum could be attained either at 
$\alpha=\sqrt{2}$, or at such $\alpha$ that 
$\frac{\mathrm{d}}{\mathrm{d}\alpha}h(\alpha,\gamma(\alpha))=0$. Note that
$\frac{\mathrm{d}}{\mathrm{d}\alpha}h(\alpha,\gamma(\alpha))=h_\alpha(\alpha,\gamma(\alpha))+\gamma'(\alpha)\cdot
 h_\gamma(\alpha,\gamma(\alpha))$, whereas $h_\gamma(\alpha,\gamma(\alpha))=0$. 
Thus we get $h_\alpha(\alpha,\gamma(\alpha))=0.$ Now we can write the following 
system of equations:

\begin{align*}
\begin{cases}
h_\gamma(\alpha,\gamma(\alpha)) &=0 \\
h_\alpha(\alpha,\gamma(\alpha)) &=0
\end{cases}
\Leftrightarrow
\begin{cases}
-\frac{\ln(2\gamma)}{\alpha^2}+\frac{1}{2}\gamma(1-\gamma) &=0 \\
\frac{1}{\alpha\gamma}-\alpha\gamma+\frac{\alpha}{2}  &=0
\end{cases}
\Leftrightarrow
\begin{cases}
\alpha^2 & = \frac{2\ln(2\gamma)}{\gamma(1-\gamma)} \\
\alpha^2 & = \frac{1}{\gamma(\gamma-0.5)}
\end{cases}
\end{align*}

Solving numerically this system of equations we find a unique solution:  
$\alpha=2.200$ and $\gamma(\alpha)=0.769$ with
$h(\alpha,\gamma)=0.391$. This number is smaller than $h(\sqrt{2},1) = 
\frac{\ln(2)}{\sqrt{2}}$. Therefore, $\alpha=2.2$, $\gamma=0.769$  is the 
optimal solution to \eqref{eq:min_max_alpha_gamma}. The resulting optimal 
regret of MWA is $\frac{0.391}{\sqrt{\delta}}$
with the optimal parameter $\eta\sim 2.2\times\sqrt{\delta}.$
\end{proof}

\begin{remark}
The regret $\frac{0.391}{\sqrt{\delta}}$ of $\algsingle$ in the geometric 
horizon model is by $10.6\%$ larger than the regret of 
$\frac{1}{2\sqrt{2\delta}}$ of the optimal algorithm (\cite{GPS16}) for $k=2$.
\end{remark}

%%%%%%%%%%%%%%%%%%%%%%%%%%%%%%%%%%%%%%%%%%%%%%%%%%%%%%%%%%%%%%%%%%%%%%%%%%%%%%%%%%%%%%%%%%%%%%%%%%%%%%%%%%%%%%%%%%%%%%%%%%%%%%%%%%%%%%%%%%%%%%%%%%%%%%%%%%%%%%%%

\subsection{Improved regret lower bound for k experts.}
\label{app:mwa}
In this section, we derive new regret lower bound for $\algsingle$ in the 
geometric horizon model.
We show that for any number of experts $k$, as $\delta\to 0$, $\algsingle$  
cannot obtain a regret smaller than $\frac{1}{2}\sqrt{\frac{\ln(k)}{2\delta}}$. 
This is an improvement over the previously best known lower bound of of 
$\sqrt{\frac{T\log_2(k)}{16}}$ in the finite horizon model and its 
corresponding implication (it's not even clear if there is any) in the 
geometric horizon model. 

One obstacle in directly generalizing the adversary for $k=2$ experts is that 
there is no immediate generalization of the looping phase after $\ell$ steps 
of straight-line phase have passed, when $k \ge 3$. On the other hand, we still 
manage to get a lower bound which is only factor $2$ away from the upper bound 
by employing the looping and straight strategies separately. 

\begin{theorem}
\label{thm:mwa_lb}
In geometric horizon, $R_k(\algsingle, \{D_\ell,D_s\}) \geq 
\frac{1}{2}\sqrt{\frac{\ln k}{2\delta}} \geq 
\frac{1}{2}R_k(\algsingle,\advuniv)$.
\end{theorem}
\begin{proof}
We begin with the case of $2$ experts and then generalize the proof. We show 
that the looping and staright-line primitives in combination yield the desired 
lower bound.  In the description below, $\eta$ refers to the parameter in the 
exponent, as described in the paragraph close to the beginning of this section. 
The main idea is to show that different regimes of $\eta$ are rendered 
ineffective by different primitives.

\paragraph{Straight-line and looping primitives.} A quick reminder of the 
straight-line and looping primitives: 

\begin{enumerate}[leftmargin =*]
\item The {\em straight-line primitive} picks an arbitrary expert and always 
keeps advancing that expert by $1$ in each step. Thus the lag between the 
leading and all lagging 
experts keeps 
monotonically increasing by one in each step. 
\item The {\em looping primitive} gives labels $A$ and $B$ to the two experts, 
and 
advances expert number $1$ in odd-numbered steps and the other expert in 
even-numbered steps. In effect, the difference between these two experts loops 
between $0$ 
and $1$. 
\end{enumerate}

Remarkably, these two simple primitives generate regrets of 
$\frac{\ln(k)}{\eta}$ and $\frac{\eta}{8\delta}$ in geometric model and that 
match exactly with the standard upper bound analysis for MWA ($\algsingle$). 
The only 
difference between the upper and lower bounds is that in the former, one gets 
the {\em sum} of these two regret terms, but in our lower bound, we get only 
the {\em max} of these two regret terms. Thus, our lower bound is exactly a 
factor $2$ away from the known upper bound of $\sqrt{\frac{\ln(k)}{2\delta}}$.

For convenience, we write $e^{\eta} = \tau = 1 + \alpha\sqrt{\delta}$. The two 
primitives/adversaries together place strong bounds on what $\alpha$ should be: 
they imply that $\alpha = \Theta(1)$.  We show this in $2$ steps: first we show 
that $\alpha = O(1)$, and then show that $\alpha = \Theta(1)$.
%we show that $\alpha > 0$, then
\begin{enumerate}[leftmargin=*]
\item The looping adversary forces $\alpha$ to be $O(1)$. The regret of MWA on 
a looping adversary is given by
$$\sum_{d=0}^{\infty}(1-\delta)^{2d+1}\left[\frac{1}{2}-\frac{1-\delta}{1+\tau}\right] = \frac{1-\delta}{1-(1-\delta)^2}\left[\frac{1}{2}-\frac{1-\delta}{1+\tau}\right]
\sim \frac{1}{2\delta}\left[\frac{1}{2}-\frac{1-\delta}{1+\tau}\right].$$ Since MWA's regret upper bound in the geometric horizon model is
$\Theta\left(\sqrt{\frac{1}{\delta}}\right)$, our lower bound on the regret $\frac{1}{2\delta}\left[\frac{1}{2}-\frac{1-\delta}{1+\tau}\right]=
\frac{1}{2\delta}\left[\frac{\alpha\sqrt{\delta}+2\delta}{2(2+\alpha\sqrt{\delta})}\right]$ is at most $\Theta\left(\sqrt{\frac{1}{\delta}}\right)$. This implies that $\alpha = O(1)$.

%Finally, we return to the straight-line adversary again to show that $\alpha = \Theta(1)$.
\item We use straight-line adversary to show that $\alpha = \Theta(1)$. We argue that when $\alpha = o(1)$, the straight-line adversary will result in the regret of $\omega(\sqrt{\frac{1}{\delta}})$. In fact, we don't even need to consider $\alpha = o(1)$. Just assume that $\alpha$ is a very small constant $c$ independent of $\delta$. Then $\tau = 1+\alpha\sqrt{\delta} = 1+c\sqrt{\delta}$. Consider running the straight-line adversary for just $\frac{1}{c\sqrt{\delta}}$ steps. For any $d  = o(\frac{1}{\delta})$, we have $(1-\delta)^{d+1} \sim 1$. Also, for all $d \leq \frac{1}{c\sqrt{\delta}}$, we have $\tau^d = (1+c\sqrt{\delta})^{d} \leq e$. Thus the straight-line adversary running till $\frac{1}{c\sqrt{\delta}}$ has a regret of

$$\sum_{d = 0}^{\frac{1}{c\sqrt{\delta}}} (1-\delta)^{d+1} \frac{1}{\tau^d + 1}\sim \sum_{d = 0}^{\frac{1}{c\sqrt{\delta}}} \frac{1}{\tau^d + 1} \geq \sum_{d = 0}^{\frac{1}{c\sqrt{\delta}}} \frac{1}{e+1} = \frac{1}{e+1}\left(1+\frac{1}{c\sqrt{\delta}}\right).$$
As $c$ gets smaller, the regret becomes very large, and hence there is a lower bound on $c$. This shows that $\alpha = \Theta(1)$.
\end{enumerate}

\paragraph{Regret calculation for $k$ experts.} We are now ready to show the 
lower bound on the regret for arbitrary number $k$ of experts. While one could 
do something more sophisticated for $k$ experts, we consider just the looping 
and straight line adversaries similar to the case of $k=2$ experts. The 
straight-line adversary for arbitrary $k$ fixes on an arbitrary expert and 
advances just that expert by $1$ in each step till the game dies. The looping 
adversary splits the set of $k$ experts into two teams $A$ and $B$ of equal 
size (if $k$ is odd, the sets are of size $(k-1)/2$ each and it won't matter 
for asymptotics), and all experts in team $A$ are advanced in odd numbered 
rounds, and those in $B$ are advanced in even numbered rounds. It is immediate 
to see that the looping adversary gets a 
regret of exactly what it obtained for the case of $2$ experts, namely, 
$\frac{1}{2\delta}\left[\frac{1}{2}-\frac{1-\delta}{1+\tau}\right]\sim\frac{\alpha}{8\sqrt{\delta}}$
(the regret is only higher when $k$ is odd). 
%this is because with more than $2$ experts now, MWA only has a higher chance 
%of making a mistake by picking an expert apart from the $2$ looping experts.
The straight line adversary generates a higher regret, and it amounts to 
$\frac{\ln(k)}{\alpha \sqrt{\delta}}$. To see this, note that with a lead of 
$d$, MWA follows the leading expert with probability 
$\frac{\tau^d}{\tau^d+k-1}$, and each of the $k-1$ other experts with 
probability $\frac{1}{\tau^d+k-1}$. This generates a regret of 
$\sum_{d=0}^{\infty}(1-\delta)^{d+1}\frac{k-1}{\tau^d+k-1}$. As in the case of 
$2$ experts, we lower bound this regret by stopping the straight line adversary 
after $\ell$ steps where $\ell = \omega\left(\sqrt{\frac{1}{\delta}}\right)$, 
but $\ell = o(\frac{1}{\delta})$. Thus, the straight-line adversary generates a 
regret of at least
\begin{align}
\label{eq:mwa_k_experts_straight}
\sum_{d = 0}^{\ell-1} (1-\delta)^{d+1} \frac{k-1}{\tau^d + k-1} &\sim \sum_{d = 0}^{\ell-1} \frac{k-1}{e^{\eta d} + k-1}\sim \int_0^{\ell}\frac{k-1}{e^{\eta x}+k-1}\,\mathrm{d}x \nonumber\\
& =\frac{1}{\eta}\int_{0}^{\ell\eta}\left[1-\frac{e^{y}}{e^{y}+k-1}\right]\,\mathrm{d}y
=\frac{1}{\eta}\left[y-\ln(e^{y}+k-1)\right]_{y=0}^{y=\ell\eta}\nonumber\\
&  = \frac{1}{\eta}\ln\left(k\frac{e^{\ell\eta}}{e^{\ell\eta}+k-1}\right)
\sim \frac{\ln(k)}{\eta}. 
\end{align}
The looping adversary, as already mentioned, obtains a regret of
\begin{align*}
%\label{eqn:looping}
\frac{1}{2\delta}\left[\frac{1}{2}-\frac{1-\delta}{1+\tau}\right]=
\frac{1}{2\delta}\left[\frac{\alpha\sqrt{\delta}+2\delta}{2(2+\alpha\sqrt{\delta})}\right]
\sim \frac{1}{2\delta}\frac{\alpha\sqrt{\delta}}{4}
= \frac{\alpha}{8\sqrt{\delta}}.
\end{align*}

Thus we get a regret lower bound of $\max\left(\frac{\ln(k)}{\alpha\sqrt{\delta}}, \frac{\alpha}{8\sqrt{\delta}}\right)$, which is minimized when $\alpha = \sqrt{8 \ln(k)}$ giving a lower bound of $\frac{1}{2}\sqrt{\frac{\ln(k)}{2\delta}}$.
\end{proof}

\end{document}